\documentclass{article}
\usepackage{fullpage}
\usepackage{amssymb}
\usepackage{amsmath}
\usepackage{amsthm}
\usepackage{mathtools}
\usepackage{float}
\usepackage{algorithm}
\usepackage{algorithmic}
\usepackage{xcolor}
\usepackage{graphicx}
\usepackage{authblk}

\newcommand{\pp}[3]{\frac{\partial^{#1}{#2}}{\partial{#3}^{#1}}}
\newcommand{\argmax}{\operatornamewithlimits{argmax}}
\newcommand{\bs}[1]{\boldsymbol{#1}}
\DeclarePairedDelimiterX{\infdivx}[2]{(}{)}{#1\;\delimsize\|\;#2}
\DeclarePairedDelimiterX{\condexp}[2]{[}{]}{#1\;\delimsize\vert\;#2}

\newtheorem{theorem}{Theorem}
\newtheorem{lemma}{Lemma}

\title{PAC-Bayesian Lifelong Learning For Multi-Armed Bandits}

\author[a]{Hamish Flynn}
\author[a]{David Reeb}
\author[b]{Melih Kandemir}
\author[c]{Jan Peters}

\begin{small}
\affil[a]{Bosch Center for Artificial Intelligence, Renningen, Germany}
\affil[b]{Department of Mathematics and Computer Science, University of Southern Denmark, Odense, Denmark}
\affil[c]{Intelligent Autonomous Systems, Technische Universit\"at Darmstadt, Germany}
\end{small}

\date{}

\begin{document}

\maketitle

\begin{abstract}
We present a PAC-Bayesian analysis of lifelong learning. In the lifelong learning problem, a sequence of learning tasks is observed one-at-a-time, and the goal is to transfer information acquired from previous tasks to new learning tasks. We consider the case when each learning task is a multi-armed bandit problem. We derive lower bounds on the expected average reward that would be obtained if a given multi-armed bandit algorithm was run in a new task with a particular prior and for a set number of steps. We propose lifelong learning algorithms that use our new bounds as learning objectives. Our proposed algorithms are evaluated in several lifelong multi-armed bandit problems and are found to perform better than a baseline method that does not use generalisation bounds.
\end{abstract}

\section{Introduction}
\label{sec:intro}

Lifelong machine learning \cite{thrun1995lifelong} is a framework in which a system continually observes new learning tasks and attempts to use its experience with previous tasks to perform new tasks more efficiently. Baxter \cite{baxter2000model} proposed a formal model of lifelong learning in which the problem can be studied. In this model, each task is a supervised learning problem that comprises an unknown input-label distribution and several input-label pairs sampled independently from this distribution. The goal of each task is to select a hypothesis that accurately predicts the labels of newly sampled inputs. The goal of the lifelong learning system is to learn an inductive bias using data from several tasks that will allow a learning algorithm to select a hypothesis that makes accurate predictions on a new task. The degree to which tasks are related is specified by a task environment. This is a fixed unknown distribution over a set of possible input-label distributions that determines which tasks are likely to appear next. It is often assumed that every new task is sampled independently from the task environment.

For example, the goal of each task could be to predict whether various flu treatments are effective against a particular flu strain. The task environment could be a distribution over flu strains. Having seen examples from several strains, a lifelong learning system may discover that, against all of these strains, treatment A is effective. The learned inductive bias may then favour hypotheses where treatment A is predicted to be effective.

In this work, we present a PAC-Bayesian \cite{shawe1997pac} \cite{mcallester1998some} analysis of the lifelong learning problem. \cite{pentina2014pac}, \cite{amit2018meta} and \cite{rothfuss2021pacoh} have proposed PAC-Bayesian generalisation bounds that apply to Baxter's model. They quantify the difference between the average error on a set of observed data sets and the expected error on a new task sampled from the task environment. Lifelong learning algorithms can use these PAC-Bayesian bounds to identify prior distributions over hypotheses that are expected to generalise well to new tasks. These existing PAC-Bayesian lifelong learning bounds only apply to lifelong supervised learning. However, it has been shown that learning from multiple tasks can also be beneficial for multi-armed bandits \cite{azar2013sequential} \cite{soare2014multi} and reinforcement learning \cite{lazaric2011transfer} \cite{brunskill2013sample} \cite{deramo2020sharing}. This motivates the development of PAC-Bayesian bounds for lifelong learning that apply to sequential decision problems.

We provide PAC-Bayesian generalisation bounds for lifelong learning in an extension of Baxter's model \cite{baxter2000model}, where each task is a multi-armed bandit (MAB) problem. In a MAB problem, there are $K$ actions (or arms), each associated with an unknown reward distribution. There is a fixed number of rounds $m$ and in each round, an action/arm is selected and a reward is sampled from the corresponding reward distribution. The goal is to select actions that maximise the cumulative sum of rewards. Returning to the flu treatment example, the goal of each task could be to choose the best flu treatments for a small sequence of patients. The task environment could be a distribution over flu strains, each one with its own set of reward distributions. A lifelong learning system may learn to start each new task by trying out a treatment that worked well against many of the previous flu strains.

Switching from supervised learning tasks to MAB tasks introduces several challenges. First of all, we no longer have data sets for each task that are i.i.d.. The data set for each task includes actions selected by the lifelong learning system. In general, these actions are dependent on previous training data from the current task and from previous tasks. Next, there is the problem of limited feedback. In supervised learning problems, any predicted label can be evaluated on every input-label pair in the training data set. In multi-armed bandit tasks, predicted actions can only be evaluated using observed action-reward pairs that contain the predicted action. Data dependence and limited feedback make quantifying our uncertainty about the unknown task environment and the unknown reward distributions for each task more difficult.

The contributions of this work are as follows:
\begin{enumerate}
\item We derive the first PAC-Bayesian generalisation bounds for lifelong learning of multi-armed bandit tasks.
\item We propose lifelong learning algorithms that use the new bounds as a learning objectives, and test them in three lifelong MAB problems.
\end{enumerate}

\section{PAC-Bayesian Bounds}

This section provides a brief introduction to PAC-Bayesian analysis \cite{shawe1997pac} \cite{mcallester1998some}. As an example, we present the PAC-Bayesian Bernstein inequality by \cite{seldin2012bern}, applied to the MAB problem. Consider a MAB problem with a discrete set of actions $\mathcal{A}$ and a reward distribution $\rho(r|a)$ that is conditioned on the action chosen. Assume that rewards are always between 0 and 1. Let $D = \{(a_i, r_i)\}_{i=1}^{m}$ be a training data set containing $m$ action-reward pairs and let $D^{:j}$ denote the first $j$ action-reward pairs in $D$. Each action $a_i$ is sampled from a behaviour policy $b_i$, which can depend on all previous observations. Let $r_i$ be a sample from the reward distribution $\rho(r|a_i)$. Let the expected reward for an action $a$ be defined as:

\begin{equation*}
R(a) = \mathop{\mathbb{E}}_{r \sim \rho(r|a)}\left[r\right].
\end{equation*}

Given a training data set $D$, the expected reward can be approximated by an unbiased estimator called the importance-weighted empirical reward. For an action $a$ and dataset $D$, this is defined as:

\begin{equation*}
\widehat{R}(a, D) = \frac{1}{m}\sum_{i=1}^{m}\frac{1}{b_{i}(a_{i}|D^{:i-1})}r_{i}\mathbb{I}\{a_{i} = a\}.
\end{equation*}

For a distribution $Q$ over actions, let $R(Q) = \mathbb{E}_{a \sim Q}[R(a)]$ and $\widehat{R}(Q, D) = \mathbb{E}_{a \sim Q}[\widehat{R}(a, D)]$. The aim is to find a $Q$ that maximises the expected reward $R(Q)$. However, if $\rho(r|a)$ is unknown, only $\widehat{R}(Q, D)$ can be computed. The PAC-Bayesian Bernstein bound by \cite{seldin2012bern} can be used to upper bound the difference between $\widehat{R}(Q, D)$ and $R(Q)$. From the upper bound on this difference, one can obtain a lower bound on $R(Q)$ that consists of only observable quantities. Theorem \ref{thm:seldin1} states the resulting lower bound, and is a restatement of Theorem 1 of \cite{seldin2012bern}.

\begin{theorem}[\cite{seldin2012bern}]
\label{thm:seldin1}
Let $\{b_1, b_2, \dots\}$ be any sequence of sampling distributions that are bounded below by $\{\epsilon_1, \epsilon_2, \dots\}$ (meaning $b_m(a) \geq \epsilon_m$ for all $a \in \mathcal{A}$ and $m \geq 1$). Let $\{P_1, P_2, \dots\}$ be any sequence of reference distributions over $\mathcal{A}$, such that $P_m$ is independent of $D^{:m}$ (but can depend on $m$). Let $\{\lambda_1, \lambda_2, \dots\}$ be any sequence of positive numbers that satisfy:

\begin{equation*}
\lambda_m \leq \epsilon_m.
\end{equation*}

Then for all possible distributions $Q_m$ over $\mathcal{A}$ given $m$ and for all $m \geq 1$ simultaneously with probability greater than $1 - \delta$

\begin{equation*}
R(Q_m) \geq \widehat{R}(Q_m, D^{:m}) - \frac{D_{\mathrm{KL}}(Q_m||P_m) + 2\mathrm{ln}(m+1) + \mathrm{ln}(1/\delta)}{m\lambda_m} - \frac{(e - 2)\lambda_m}{\epsilon_m}.
\end{equation*}
\end{theorem}

This bound states that, with probability at least $1 - \delta$ and for all rounds $m \geq 1$, the expected reward is lower bounded by the empirical reward for $Q_m$ plus a complexity term containing the KL divergence between $Q_m$ and the reference distribution $P_m$, which must be chosen before observing the training data. This lower bound can be used as an objective function for finding a $Q_m$ that maximises $R(Q_m)$. It also motivates learning useful priors, for example from previous tasks. If $P_m$ is chosen such that there is a $Q_m$ with high empirical reward and where $D_{\mathrm{KL}}(Q_m||P_m)$ is small, then the lower bound on $R(Q_m)$ will be large.

\section{PAC-Bayesian Bounds for Lifelong Multi-Armed Bandits}
\label{sec:main}

\subsection{Problem Setup}

We represent a multi-armed bandit (MAB) task as a couple

\begin{equation*}
T_i = (\mathcal{A}, \rho_i),
\end{equation*}

where $\mathcal{A}$ is a finite set of actions and $\rho_i(r|a)$ is a distribution over rewards $r$ that is conditioned on the action $a \in \mathcal{A}$. We assume that rewards are always between 0 and 1 and that tasks are sampled i.i.d. from an environment $\mathcal{T}$.

In each task, the lifelong learning algorithm chooses a sequence of behavior policies $\{b_{ij}\}_{j=1}^{m}$ that is used to sample the actions of a data set $D_i = \{(a_{ij}, r_{ij})\}_{j=1}^{m}$. $b_{ij}$ is a distribution over actions that can depend on all previously observed training data $D_1, \dots, D_{i-1}, D_i^{:j-1}$, where $D_i^{:j-1} = \{(a_{ik}, r_{ik})\}_{k=1}^{j-1}$ is the first $j-1$ action reward pairs for task $i$. We assume that the training data set for every task contains $m$ action-reward pairs. Therefore, the elements of each training data set are distributed as follows:

\begin{equation*}
a_{ij} \sim b_{ij}(a|D_1, \dots, D_{i-1}, D_i^{:j-1}), \qquad r_{ij} \sim \rho_i(r|a_{ij}).
\end{equation*}

In the interest of concise notation, we will let $b_{ij}(a)$ denote the probability mass functions of the behaviour policies conditioned on all previous training data. Let $\mathcal{M}(\mathcal{A})$ denote the set of all probability distributions over $\mathcal{A}$ and let $\mathcal{D}$ denote the set of all possible training data sets. To solve each task, a deterministic learning algorithm $A: \mathcal{D} \times \mathcal{M}(\mathcal{A}) \to \mathcal{M}(\mathcal{A})$ takes a data set $D_i$ and a prior $P$ as inputs and produces a posterior $Q = A(D_i, P)$. We refer to $P$ as a prior and $Q$ as a posterior since $P$ must be chosen before observing any training data from the current task, whereas $Q$ can be chosen afterwards. For task $i$, the expected reward for an action $a$ is:

\begin{equation*}
R_i(a) = \mathop{\mathbb{E}}_{r \sim \rho_i(r|a)}\left[r\right].
\end{equation*}

Using the training set $D_i$, the expected reward can be estimated with an importance-weighted empirical reward estimate:

\begin{equation*}
\widehat{R}_i(a, D_i) = \frac{1}{m}\sum_{j=1}^{m}\frac{1}{b_{ij}(a_{ij})}r_{ij}\mathbb{I}\{a_{ij} = a\}.
\end{equation*}

Due to the potentially very large importance sampling weights $1/b_{ij}(a_{ij})$, this reward estimate can have very high variance depending on the choice of $b_{ij}$. This can be addressed by constraining $b_{ij}$ or by clipping the importance sampling weights to the range $[0, 1+\tau]$. Define the clipped importance-weighted reward estimate as:

\begin{equation*}
\widehat{R}_i^{\tau}(a, D_i) = \frac{1}{m}\sum_{j=1}^{m}\min\left(\frac{1}{b_{ij}(a_{ij})}, 1 + \tau\right)r_{ij}\mathbb{I}\{a_{ij} = a\}.
\end{equation*}

The importance-weighted estimate is an unbiased estimate of the expected reward, and since clipping the importance sampling weights cannot increase the value of reward estimate, we have that:

\begin{equation*}
\mathbb{E}_{D_i}\left[\widehat{R}_i(a, D_i)\right] = R_i(a) \qquad \text{and} \qquad \mathbb{E}_{D_i}\left[\widehat{R}_i^{\tau}(a, D_i)\right] \leq R_i(a).
\end{equation*}

Let $R_i(Q)$, $\widehat{R}_i(Q, D_i)$ and $\widehat{R}_i^{\tau}(Q, D_i)$ denote the expected values of $R_i(a)$, $\widehat{R}_i(a, D_i)$ and $\widehat{R}_i^{\tau}(a, D_i)$ when $a$ is sampled from $Q$.

The inductive bias that the lifelong learning system must learn is the prior $P$. To derive PAC-Bayesian bounds for lifelong learning, we require the notions of a hyperprior and a hyperposterior. The hyperprior $\mathcal{P}$ and hyperposterior $\mathcal{Q}$ are both distributions over the set of possible priors. The hyperprior must be chosen before we observe data from any tasks, whereas the hyperposterior can be chosen afterwards. A well-chosen hyperposterior will assign high probability mass/density to priors that result in posteriors with high expected reward on new tasks. We measure the performance of a hyperposterior by the marginal transfer reward, which is the expected value of the average reward obtained in the first $m$ rounds on a new task $T_{n+1}$ when using the base learning algorithm $A$ with a prior $P$ sampled from $\mathcal{Q}$.

\begin{equation*}
\mathcal{R}(\mathcal{Q}) = \mathop{\mathbb{E}}_{P \sim \mathcal{Q}}\left[\mathop{\mathbb{E}}_{(T_{n+1}, D_{n+1})}\left[\frac{1}{m}\sum_{j=1}^{m}R_{n+1}(A(D_{n+1}^{:j-1}, P))\right]\right].
\end{equation*}

It is called marginal because the expectation with respect to $(T_{n+1}, D_{n+1})$ is not conditioned on the observed tasks and training data sets $(T_{1}, D_{1}), \dots, (T_{n}, D_{n})$. Since the task environment is unknown, the marginal transfer reward cannot be calculated and so we cannot directly maximise it with respect to the hyperposterior. Therefore, we will instead maximise lower bounds on the marginal transfer reward. The PAC-Bayesian bounds in this work bound the difference between the marginal transfer reward and the following empirical estimates, called the (clipped) empirical multi-task reward
\begin{align*}
\widehat{\mathcal{R}}(\mathcal{Q}) = \mathop{\mathbb{E}}_{P \sim \mathcal{Q}}\left[\frac{1}{n}\sum_{i=1}^{n}\frac{1}{m}\sum_{j=1}^{m}\widehat{R}_i(A(D_i^{:j-1}, P), D_i)\right],\\
\widehat{\mathcal{R}}_{\tau}(\mathcal{Q}) = \mathop{\mathbb{E}}_{P \sim \mathcal{Q}}\left[\frac{1}{n}\sum_{i=1}^{n}\frac{1}{m}\sum_{j=1}^{m}\widehat{R}_i^{\tau}(A(D_i^{:j-1}, P), D_i)\right].
\end{align*}

We split the task of bounding the difference between the marginal transfer reward and the empirical multi-task reward into two steps. First, we bound the difference between the marginal transfer reward and an intermediate quantity called the expected multi-task reward. The expected multi-task reward is defined as:

\begin{equation*}
\widetilde{\mathcal{R}}(\mathcal{Q}) = \mathop{\mathbb{E}}_{P \sim \mathcal{Q}}\left[\frac{1}{n}\sum_{i=1}^{n}\frac{1}{m}\sum_{j=1}^{m}R_i(A(D_i^{:j-1}, P))\right].
\end{equation*}

We then bound the difference between the expected and empirical multi-task reward. Finally, we add this difference to the first difference to obtain a lower bound on the marginal transfer reward.

\subsection{Simplifying Assumptions}\label{simple}

Recall from section 3.1, that we assume the tasks $T_1, T_2, \dots$ are sampled i.i.d. from the task environment. This means that the reward distributions $\rho_i$ are independent of each other, but the task data sets $D_i$, which are the only observable quantities, are still dependent on each other since each behaviour policy may depend on data from previous tasks.

We make two simplifying assumptions that make bounding the marginal transfer reward more feasible. Only the first assumption is necessary to derive our bounds and to select the hyperposterior that maximises the bounds. If we want to evaluate the bounds, the second assumption allows us to deal with a term in our bounds that cannot easily be computed.

Our first assumption restricts how the expected reward for each task can be related to the expected reward of previous tasks. We assume that the marginal expected reward for any prior $P$ and any number of training samples $j$ is greater for task $n + k$ than it is for task $n$. More precisely, it is assumed that for all $1 \leq j \leq m$ and $n, k \in \mathbb{N}$

\begin{equation}
\mathop{\mathbb{E}}_{(T_{n+k}, D_{n+k})}\left[R_{n+k}(A(D_{n+k}^{:j}, P))\right] \geq \mathop{\mathbb{E}}_{(T_{n}, D_{n})}\left[R_{n}(A(D_{n}^{:j}, P))\right].\label{eqn:const_exp_assump}
\end{equation}

Since we have assumed that the tasks $T_1, T_2, \dots$ are i.i.d., the expected reward for tasks $n$ and $n+k$ can only be different if $D_n$ and $D_{n+k}$ have different distributions, which happens when the behaviour policies for each task are different. Therefore, this assumption requires that the posterior resulting from data sampled with the behaviour policies for task $n+k$ does not have lower expected reward than the posterior resulting from running the behaviour policies for task $n$. Our proposed algorithms always use behaviour policies that depend greatly on the current hyperposterior. Since the hyperposterior is continually being improved after each new task is observed, we expect this assumption to hold.

Our second assumption restricts how strongly the expected reward for each task can depend on data from previous tasks. We assume that with high probability (over the set of observed tasks and datasets $(T_{1}, D_{1}), \dots, $ $(T_{n-1}, D_{n-1})$)

\begin{align}
\lim_{n \to \infty}\bigg(&\mathop{\mathbb{E}}_{(T_n, D_n)}\condexp*{\frac{1}{m}\sum_{j=1}^{m}R_n(A(D_n^{:j-1}, P))}{D_1, \dots, D_{n-1}}\nonumber\\
&- \mathop{\mathbb{E}}_{(T_n, D_n)}\left[\frac{1}{m}\sum_{j=1}^{m}R_n(A(D_n^{:j-1}, P))\right]\bigg) = 0.\label{eqn:cond_exp_assump}
\end{align}

In other words, once a sufficiently large number of tasks and data sets have been observed, the specific set of observed data sets has a negligible effect on the expected value of the average of the first $m$ rewards, for a high proportion of possible sets of observed tasks and data sets. This assumption holds whenever the sequence of expected rewards converges to the same value. Since we have assumed that the tasks are sampled from a fixed distribution, we expect that this is the case for our setting.

\subsection{Main Results}

The derivation of our main results can be split into two parts. Firstly, in Lemma \ref{lem:transfer_bound}, we obtain a bound on the difference between the marginal transfer reward and the expected multi-task reward. This bound quantifies our uncertainty about the environment $\mathcal{T}$, given a sample of tasks $T_1, \dots, T_n$. Secondly, in Lemma \ref{lem:multi_risk_bern} and Lemma \ref{lem:multi_risk_trunc}, we obtain bounds on the difference between the expected multi-task reward and the empirical multi-task reward. Our main results are obtained by combining these lemmas.

We will begin by stating some auxiliary lemmas that will be used in the proof of Lemma \ref{lem:transfer_bound}. The first is a change of measure inequality from \cite{donsker1975asymptotic} and is referred to as the compression lemma in \cite{banerjee2006bayesian}. We will also refer to it as the compression lemma.

\begin{lemma}[Compression Lemma \cite{donsker1975asymptotic,banerjee2006bayesian}]
\label{lem:compression}
For any measurable function $f(x)$ on $\mathcal{X}$ and any distributions $q$ and $p$ on $\mathcal{X}$, the following inequality holds

\begin{equation*}
\mathop{\mathbb{E}}_{x \sim q}\left[f(x)\right] \leq D_{\mathrm{KL}}(q||p) + \mathrm{ln}\left(\mathop{\mathbb{E}}_{x \sim p} \left[e^{f(x)}\right] \right)
\end{equation*}
\end{lemma}

The second auxiliary lemma is Hoeffding's lemma \cite{hoeffding1994probability}.

\begin{lemma}[Hoeffding's Lemma \cite{hoeffding1994probability}]
\label{lem:hoeffding}
Let $X$ be a real-valued random variable such that $Pr(X \in [a,b]) = 1$. For all $\lambda \in \mathbb{R}$

\begin{equation*}
\mathbb{E}\left[e^{\lambda(\mathbb{E}[X] - X)}\right] \leq e^{\frac{\lambda^2}{8}
(b - a)^2}
\end{equation*}
\end{lemma}

Now we are ready to state our bound on the difference between the marginal transfer reward and the expected multi-task reward.

\begin{lemma}
\label{lem:transfer_bound}

If condition (\ref{eqn:const_exp_assump}) is satisfied, then for any hyperprior $\mathcal{P}$, any $\lambda > 0$ and any $\delta \in (0, 1]$, inequality (\ref{eqn:lem1:lower}) holds with probability at least $1 - \delta$ over the tasks and training sets $(T_1, D_1), \dots, (T_n, D_n)$ and for all hyperposteriors $\mathcal{Q}$

\begin{equation}
\mathcal{R}(\mathcal{Q}) \geq \widetilde{\mathcal{R}}(\mathcal{Q}) - c_n - \frac{1}{\lambda}\left(D_{\mathrm{KL}}(\mathcal{Q}||\mathcal{P}) + \frac{\lambda^2}{8n} + \mathrm{ln}(1/\delta)\right),\label{eqn:lem1:lower}
\end{equation}

where

\begin{align*}
c_n = \sup_{P}\bigg\{&\frac{1}{n}\sum_{i=1}^{n}\mathop{\mathbb{E}}_{(T_n, D_n)}\condexp*{\frac{1}{m}\sum_{j=1}^{m}R_n(A(D_n^{:j-1}, P))}{D_1, \dots, D_{n-1}}\\
&- \frac{1}{n}\sum_{i=1}^{n}\mathop{\mathbb{E}}_{(T_n, D_n)}\left[\frac{1}{m}\sum_{j=1}^{m}R_n(A(D_n^{:j-1}, P))\right]\bigg\}
\end{align*}
\end{lemma}

\begin{proof}
Throughout this proof, let $R_{i,j} = R_{i}(A(D_{i}^{:j-1}, P))$. If the assumption in Equation (\ref{eqn:const_exp_assump}) holds, then the marginal transfer reward $\mathcal{R}(\mathcal{Q})$ can be lower bounded as follows:

\begin{equation*}
\mathop{\mathbb{E}}_{P \sim \mathcal{Q}}\left[\mathop{\mathbb{E}}_{(T_{n+1}, D_{n+1})}\left[\frac{1}{m}\sum_{j=1}^{m}R_{n+1,j}\right]\right] \geq \mathop{\mathbb{E}}_{P \sim \mathcal{Q}}\left[\frac{1}{n}\sum_{i=1}^{n}\mathop{\mathbb{E}}_{(T_{i}, D_{i})}\left[\frac{1}{m}\sum_{j=1}^{m}R_{i,j}\right]\right].
\end{equation*}

Now, we need to upper bound the following quantity:

\begin{equation*}
\mathop{\mathbb{E}}_{P \sim \mathcal{Q}}\left[\frac{1}{n}\sum_{i=1}^{n}\frac{1}{m}\sum_{j=1}^{m}R_{i,j} - \mathop{\mathbb{E}}_{(T_{i}, D_{i})}\left[\frac{1}{m}\sum_{j=1}^{m}R_{i,j}\right]\right].
\end{equation*}

We first rewrite this difference as:

\begin{align}
\mathop{\mathbb{E}}_{P \sim \mathcal{Q}}&\left[\frac{1}{n}\sum_{i=1}^{n}\frac{1}{m}\sum_{j=1}^{m}R_{i,j} - \mathop{\mathbb{E}}_{(T_{i}, D_{i})}\condexp*{\frac{1}{m}\sum_{j=1}^{m}R_{i,j}}{D_1, \dots, D_{i-1}}\right]\label{eqn:lem3_decomp}\\
+ \mathop{\mathbb{E}}_{P \sim \mathcal{Q}}&\left[\frac{1}{n}\sum_{i=1}^{n}\mathop{\mathbb{E}}_{(T_{i}, D_{i})}\condexp*{\frac{1}{m}\sum_{j=1}^{m}R_{i,j}}{D_1, \dots, D_{i-1}} - \mathop{\mathbb{E}}_{(T_{i}, D_{i})}\left[\frac{1}{m}\sum_{j=1}^{m}R_{i,j}\right]\right]\nonumber
\end{align}

From the definition of $c_n$ in the statement of the lemma, we have that the second line of Equation (\ref{eqn:lem3_decomp}) is upper bounded by $c_n$. Next, we upper bound the first line of Equation (\ref{eqn:lem3_decomp}). Using the compression lemma, for any $\lambda > 0$, we have that:

\begin{align}
&\mathop{\mathbb{E}}_{P \sim \mathcal{Q}}\left[\frac{1}{n}\sum_{i=1}^{n}\frac{1}{m}\sum_{j=1}^{m}R_{i,j} - \mathop{\mathbb{E}}_{(T_{i}, D_{i})}\condexp*{\frac{1}{m}\sum_{j=1}^{m}R_{i,j}}{D_1, \dots, D_{i-1}}\right],\label{eqn:lem1:compress}\\
&\leq \frac{1}{\lambda}\left(D_{\mathrm{KL}}(\mathcal{Q}||\mathcal{P}) + \mathrm{ln}\left(\mathop{\mathbb{E}}_{P \sim \mathcal{P}}\left[e^{\frac{\lambda}{n}\sum_{i=1}^{n}\frac{1}{m}\sum_{j=1}^{m}R_{i,j} - \mathop{\mathbb{E}}_{(T_{i}, D_{i})}\condexp*{\frac{1}{m}\sum_{j=1}^{m}R_{i,j}}{D_1, \dots, D_{i-1}}}\right]\right)\right),\nonumber
\end{align}

where $\mathcal{P}$ is another distribution over priors. Next, the exponential term in Equation (\ref{eqn:lem1:compress}) must be upper bounded. It can be rewritten as a product of exponentials. Then, for any $\delta \in (0, 1]$, using Markov's inequality with respect to expectations over $(T_1, D_1), \dots, (T_n, D_n)$, the following inequality holds with probability at least $1 - \delta$

\begin{align*}
\mathop{\mathbb{E}}_{P \sim \mathcal{P}}&\left[\prod_{i=1}^{n}e^{\frac{\lambda}{nm}\left(\sum_{j=1}^{m}R_{i,j} - \mathop{\mathbb{E}}_{(T_{i}, D_{i})}\condexp*{R_{i,j}}{D_1, \dots, D_{i-1}}\right)}\right] \leq\\
&\frac{1}{\delta}\mathop{\mathbb{E}}_{(T_1, D_1), \dots, (T_n, D_n)}\mathop{\mathbb{E}}_{P \sim \mathcal{P}}\left[\prod_{i=1}^{n}e^{\frac{\lambda}{nm}\left(\sum_{j=1}^{m}R_{i,j} - \mathop{\mathbb{E}}_{(T_{i}, D_{i})}\condexp*{R_{i,j}}{D_1, \dots, D_{i-1}}\right)}\right]
\end{align*}

If $\mathcal{P}$ does not depend on the observed tasks and training sets $(T_1, D_1), \dots$, $(T_n, D_n)$, e.g. $\mathcal{P}$ is chosen before observing any tasks, then the order of expectations can be swapped. Since each task $T_i$ is sampled i.i.d. from $\mathcal{T}$ and each training set $D_i$ depends only on the training sets that came before it, the expectation over $(T_1, D_1), \dots,$ $(T_n, D_n)$ can be factorised as follows

\begin{align*}
&\frac{1}{\delta}\mathop{\mathbb{E}}_{P \sim \mathcal{P}}\mathop{\mathbb{E}}_{(T_1, D_1), ..., (T_n, D_n)}\left[\prod_{i=1}^{n}e^{\frac{\lambda}{nm}\left(\sum_{j=1}^{m}R_{i,j} - \mathop{\mathbb{E}}_{(T_{i}, D_{i})}\condexp*{R_{i,j}}{D_1, \dots, D_{i-1}}\right)}\right] = \\
&\frac{1}{\delta}\mathop{\mathbb{E}}_{P \sim \mathcal{P}}\mathop{\mathbb{E}}_{(T_1, D_1), \dots, (T_{n-1}, D_{n-1})}\left[\mathop{\mathbb{E}}_{(T_n, D_n)}\condexp*{\prod_{i=1}^{n}e^{\frac{\lambda}{nm}\left(\sum_{j=1}^{m}R_{i,j} - \mathop{\mathbb{E}}_{(T_{i}, D_{i})}\condexp*{R_{i,j}}{D_1, \dots, D_{i-1}}\right)}}{D_1, ..., D_{n-1}}\right]
\end{align*}

Hoeffding's lemma can be used, with $a = -\frac{1}{nm}\mathop{\mathbb{E}}_{(T_{i}, D_{i})}\condexp*{\sum_{j=1}^{m}R_{i,j}}{D_1, \dots, D_{i-1}}$ and\\$b = \frac{1}{n} - \frac{1}{nm}\mathop{\mathbb{E}}_{(T_{i}, D_{i})}\condexp*{\sum_{j=1}^{m}R_{i,j}}{D_1, \dots, D_{i-1}}$, to upper bound the $n$th term in the product.

\begin{align*}
&\frac{1}{\delta}\mathop{\mathbb{E}}_{P \sim \mathcal{P}}\mathop{\mathbb{E}}_{(T_1, D_1), ..., (T_{n-1}, D_{n-1})}\left[\mathop{\mathbb{E}}_{(T_n, D_n)}\condexp*{\prod_{i=1}^{n}e^{\frac{\lambda}{nm}\left(\sum_{j=1}^{m}R_{i,j} - \mathop{\mathbb{E}}_{(T_{i}, D_{i})}\condexp*{R_{i,j}}{D_1, \dots, D_{i-1}}\right)}}{D_1, ..., D_{n-1}}\right]\\
&\leq \frac{1}{\delta}\mathop{\mathbb{E}}_{P \sim \mathcal{P}}\mathop{\mathbb{E}}_{(T_1, D_1), \dots, (T_{n-1}, D_{n-1})}\left[e^{\frac{\lambda^2}{8n^2}}\prod_{i=1}^{n-1}e^{\frac{\lambda}{nm}\left(\sum_{j=1}^{m}R_{i,j} - \mathop{\mathbb{E}}_{(T_{i}, D_{i})}\condexp*{R_{i,j}}{D_1, \dots, D_{i-1}}\right)}\right],
\end{align*}

Through alternating steps of factorisation and application of Hoeffding's lemma, we have that

\begin{align*}
\frac{1}{\delta}\mathop{\mathbb{E}}_{P \sim \mathcal{P}}&\mathop{\mathbb{E}}_{(T_1, D_1), \dots, (T_{n-1}, D_{n-1})}\left[e^{\frac{\lambda^2}{8n^2}}\prod_{i=1}^{n-1}e^{\frac{\lambda}{nm}\left(\sum_{j=1}^{m}R_{i,j} - \mathop{\mathbb{E}}_{(T_{i}, D_{i})}\condexp*{R_{i,j}}{D_1, \dots, D_{i-1}}\right)}\right]\\
&\leq \frac{1}{\delta}e^{\frac{\lambda^2}{8n}},
\end{align*}

Substituting this into Equation (\ref{eqn:lem1:compress}), we have that for any $\delta \in (0, 1]$ and any $\lambda > 0$, with probability at least $1 - \delta$

\begin{align*}
\mathop{\mathbb{E}}_{P \sim \mathcal{Q}}\left[\mathop{\mathbb{E}}_{(T_{i}, D_{i})}\condexp*{\frac{1}{m}\sum_{j=1}^{m}R_{i,j}}{D_1, \dots, D_{i-1}}\right] &\geq \mathop{\mathbb{E}}_{P \sim \mathcal{Q}}\left[\frac{1}{n}\sum_{i=1}^{n}\frac{1}{m}\sum_{j=1}^{m}R_{i,j}\right]\\
&- \frac{1}{\lambda}\left(D_{\mathrm{KL}}(\mathcal{Q}||\mathcal{P}) + \frac{\lambda^2}{8n} + \mathrm{ln}(1/\delta)\right).
\end{align*}

Combining this with Equation (\ref{eqn:lem3_decomp}) and the fact that $c_n$ upper bounds the second line of Equation (\ref{eqn:lem3_decomp}), we have that, with probability at least $1 - \delta$:

\begin{equation*}
\mathcal{R}(\mathcal{Q}) \geq \widetilde{\mathcal{R}}(\mathcal{Q}) - c_n - \frac{1}{\lambda}\left(D_{\mathrm{KL}}(\mathcal{Q}||\mathcal{P}) + \frac{\lambda^2}{8n} + \mathrm{ln}(1/\delta)\right).
\end{equation*}
\end{proof}

Next, we state our first bound on the difference between the expected multi-task reward and the empirical multi-task reward, which is based on the PAC-Bayesian Bernstein inequality for martingales \cite{seldin2012bern}. First, we state some auxiliary lemmas.

\begin{lemma}[\cite{seldin2012bern}]
\label{lem:bern}
Let $X_1, \dots, X_n$ be a martingale difference sequence (meaning $\mathbb{E}\left[X_i|X_1, \dots, X_{i-1}\right] = 0$), such that $X_i \leq c$ for all $i$ with probability 1. Let $M_n = \sum_{i=1}^{n}X_i$ be the corresponding martingale and $V_{n} = \sum_{i=1}^{n}\mathbb{E}\left[X_i^2|X_1, \dots, X_{i-1}\right]$ be the cumulative variance of this martingale. Then for any $\lambda \in [0, 1/c]$:

\begin{equation*}
\mathop{\mathbb{E}}\left[e^{\lambda M_n - (e - 2)\lambda^2V_n}\right] \leq 1.
\end{equation*}

\end{lemma}

To utilise Lemma \ref{lem:bern}, we construct a martingale difference sequence from the training data of each task. Define:

\begin{equation*}
X_{ij}(a) = \frac{1}{b_{ij}(a_{ij})}\mathbb{I}\{a_{ij}=a\}r_{ij} - R_i(a).
\end{equation*}

If we let $b_{\mathrm{min}} = \min_{i\leq n, j\leq m, a \in \mathcal{A}}(b_{ij}(a))$ and assume $b_{\mathrm{min}} > 0$, then $X_{ij}(a) \leq 1/b_{\mathrm{min}}$ for all $i$, $j$ and $a$. Next, we verify that $X_{ij}(a)$ form a martingale difference sequence. For any $i \leq n$ and $j \leq m$:

\begin{align*}
\mathbb{E}&\condexp*{X_{ij}(a)}{D_1, \dots, D_{i-1}, D_i^{:j-1}}\\
&= \mathbb{E}\condexp*{\frac{1}{b_{ij}(a_{ij})}\mathbb{I}\{a_{ij}=a\}r_{ij} - R_i(a)}{D_1, \dots, D_{i-1}, D_i^{:j-1}}\\
&= 0.
\end{align*}

Therefore $M_{ij}(a) = \sum_{k=1}^{j}X_{ik}(a)$ is a martingale that satisfies the conditions of Lemma \ref{lem:bern}. Let

\begin{equation*}
V_{ij}(a) = \sum_{k=1}^{j}\mathbb{E}\condexp*{(X_{ik}(a))^2}{D_1, \dots, D_{i-1}, D_{i}^{:k-1}}
\end{equation*}

be the cumulative variance of this martingale. We use the following upper bound on $V_{ij}(a)$ from \cite{seldin2012bern}:

\begin{lemma}[\cite{seldin2012bern}]
\label{lem:var_bound}
For any $i \leq n$, any $j \leq m$ and any $a \in \mathcal{A}$:

\begin{equation*}
V_{ij}(a) \leq \frac{j}{b_{\mathrm{min}}}
\end{equation*}
\end{lemma}

Finally, let $(\mathcal{Q}, A_k^n)$ denote the joint distribution over $(P, a_1, \dots, a_n)$ where $P \sim \mathcal{Q}$ and $a_i \sim A(D_i^{:k}, P)$ for $i = 1, \dots, n$. Similarly, let $(\mathcal{P}, P^n)$ denote the joint distribution over $(P, a_1, \dots, a_n)$ where $P \sim \mathcal{P}$ and $a_i \sim P$ for $i = 1, \dots, n$. We have that

\begin{align}
\label{eqn:kl}
D_{\mathrm{KL}}&((\mathcal{Q}, A_k^n)||(\mathcal{P}, P^n)) = \mathop{\mathbb{E}}_{P \sim \mathcal{Q}}\mathop{\mathbb{E}}_{a_i \sim A(D_i^{:k-1},P)}\left[\mathrm{ln}\frac{\mathcal{Q}(P)\prod_{i=1}^{n}A(D_i^{:k}, P)(a_i)}{\mathcal{P}(P)\prod_{i=1}^{n}P(a_i)}\right]\\
&= \mathop{\mathbb{E}}_{P \sim \mathcal{Q}}\mathop{\mathbb{E}}_{a_i \sim A(D_i^{:k-1},P)}\left[\mathrm{ln}\frac{\mathcal{Q}(P)\prod_{i=1}^{n}A(D_i^{:k}, P)(a_i)}{\mathcal{P}(P)\prod_{i=1}^{n}P(a_i)}\right]\nonumber\\
&= \mathop{\mathbb{E}}_{P \sim \mathcal{Q}}\mathop{\mathbb{E}}_{a_i \sim A(D_i^{:k-1},P)}\left[\mathrm{ln}\frac{\mathcal{Q}(P)}{\mathcal{P}(P)} + \sum_{i=1}^{n}\mathrm{ln}\frac{A(D_i^{:k}, P)(a_i)}{P(a_i)}\right]\nonumber\\
&= \mathop{\mathbb{E}}_{P \sim \mathcal{Q}}\left[\mathrm{ln}\frac{\mathcal{Q}(P)}{\mathcal{P}(P)}\right] + \mathop{\mathbb{E}}_{P \sim \mathcal{Q}}\left[\sum_{i=1}^{n}\mathop{\mathbb{E}}_{a_i \sim A(D_i^{:k}, P)}\left[\mathrm{ln}\frac{A(D_i^{:k}, P)(a_i)}{P(a_i)}\right]\right]\nonumber\\
&= D_{\mathrm{KL}}(\mathcal{Q}||\mathcal{P}) + \mathop{\mathbb{E}}_{P \sim \mathcal{Q}}\left[\sum_{i=1}^{n}D_{\mathrm{KL}}(A(D_i^{:k}, P)||P)\right]\nonumber
\end{align}

Now we are ready to state and prove our bound on the difference between $\widehat{\mathcal{R}}(\mathcal{Q})$ and $\widetilde{\mathcal{R}}(\mathcal{Q})$.

\begin{lemma}
\label{lem:multi_risk_bern}
For any hyperprior $\mathcal{P}$, any $\lambda \in [0, mb_{\mathrm{min}}]$, and any $\delta \in (0, 1]$, inequality (\ref{eqn:multib_lower}) holds with probability at least $1 - \delta$ over the training sets $\{D_1, \dots, D_n\}$ and for all hyperposteriors $\mathcal{Q}$

\begin{align}
\widetilde{\mathcal{R}}(\mathcal{Q}) &\geq \widehat{\mathcal{R}}(\mathcal{Q}) -\frac{1}{n\lambda}D_{\mathrm{KL}}(\mathcal{Q}||\mathcal{P}) - \frac{1}{nm\lambda}\mathop{\mathbb{E}}_{P \sim \mathcal{Q}}\left[\sum_{i=1}^{n}\sum_{j=1}^{m}D_{\mathrm{KL}}(A(D_{i}^{:j-1}, P)||P)\right]\nonumber\\
&- \frac{\lambda(e - 2)}{b_{\mathrm{min}} m} - \frac{1}{n\lambda}\mathrm{ln}(m/\delta).\label{eqn:multib_lower}
\end{align}
\end{lemma}

\begin{proof}

Using the compression lemma, we have that for any $k = 1, \dots, m$

\begin{align}
\mathop{\mathbb{E}}_{P \sim \mathcal{Q}, a_i \sim A(D_i^{:k-1},P)}&\left[\sum_{i=1}^{n}M_{im}(a_i) - \lambda(e-2)\sum_{i=1}^{n}V_{im}(a_i)\right]\nonumber\\
&\leq \frac{1}{\lambda}D_{\mathrm{KL}}((\mathcal{Q}, A_{k-1}^n)||(\mathcal{P}, P^n))\nonumber\\ &+\frac{1}{\lambda}\mathrm{ln}\left(\mathop{\mathbb{E}}_{P \sim \mathcal{P}, a_i \sim P}\left[e^{\lambda\sum_{i=1}^{n}M_{im}(a_i) - \lambda^2(e-2)\sum_{i=1}^{n}M_{im}(a_i)}\right]\right)\label{eqn:lem2:compressb}.
\end{align}

Now, we need to upper bound the term inside the logarithm. For any $\delta \in (0, 1]$, using Markov's inequality with respect to expectations over $D_1, \dots, D_n$, the following inequality holds with probability greater than $1 - \delta$

\begin{align*}
\mathop{\mathbb{E}}_{P \sim \mathcal{P}, a_i \sim P}&\left[\prod_{i=1}^{n}e^{\lambda\sum_{i=1}^{n}M_{im}(a_i) - \lambda^2(e-2)\sum_{i=1}^{n}V_{im}(a_i)}\right]\\
&\leq \frac{1}{\delta}\mathop{\mathbb{E}}_{D_1, \dots, D_n}\mathop{\mathbb{E}}_{P \sim \mathcal{P}, a_i \sim P}\left[\prod_{i=1}^{n}e^{\lambda\sum_{i=1}^{n}M_{im}(a_i) - \lambda^2(e-2)\sum_{i=1}^{n}V_{im}(a_i)}\right].
\end{align*}

If $\mathcal{P}$ does not depend on any of the training sets $D_1, \dots, D_n$, then the order of the expectations can be swapped. Since each training set is only dependent on the training sets that came before it, the expectation over $D_1, \dots, D_n$ can be factorised.

\begin{align*}
\frac{1}{\delta}&\mathop{\mathbb{E}}_{P \sim \mathcal{P}, a_i \sim P}\mathop{\mathbb{E}}_{D_1, \dots, D_n}\left[\prod_{i=1}^{n}e^{\lambda M_{im}(a_i) - \lambda^2(e-2)V_{im}(a_i)}\right]\\
&= \frac{1}{\delta}\mathop{\mathbb{E}}_{P \sim \mathcal{P}, a_i \sim P}\mathop{\mathbb{E}}_{D_1, \dots, D_{n-1}}\left[\mathop{\mathbb{E}}_{D_n}\condexp*{\prod_{i=1}^{n}e^{\lambda M_{im}(a_i) - \lambda^2(e-2)V_{im}(a_i)}}{D_1, \dots, D_{n-1}}\right].
\end{align*}

Using Lemma \ref{lem:bern}, the $n$th term of the product can be upper bounded.

\begin{align*}
\frac{1}{\delta}\mathop{\mathbb{E}}_{P \sim \mathcal{P}, a_i \sim P}&\mathop{\mathbb{E}}_{D_1, \dots, D_{n-1}}\left[\mathop{\mathbb{E}}_{D_n}\condexp*{\prod_{i=1}^{n}e^{\lambda M_{im}(a_i) - \lambda^2(e-2)V_{im}(a_i)}}{D_1, \dots, D_{n-1}}\right]\\
&\leq \frac{1}{\delta}\mathop{\mathbb{E}}_{P \sim \mathcal{P}, a_i \sim P}\mathop{\mathbb{E}}_{D_1, \dots, D_{n-1}}\left[1 \times \prod_{i=1}^{n-1}e^{\lambda M_{im}(a_i) - \lambda^2(e-2)V_{im}(a_i)}\right]
\end{align*}

Through alternating steps of factorisation and application of Lemma \ref{lem:bern}, we have that

\begin{equation*}
\frac{1}{\delta}\mathop{\mathbb{E}}_{P \sim \mathcal{P}, a_i \sim P}\mathop{\mathbb{E}}_{D_1, \dots, D_n}\left[\prod_{i=1}^{n}e^{\lambda M_i^m(a_i)}\right] \leq \frac{1}{\delta}.
\end{equation*}

Substituting this into Equation (\ref{eqn:lem2:compressb}), we have that for any $\delta \in (0, 1]$, any $\lambda \in [0, b_{\mathrm{min}}]$ and any $k \in \{1, \dots, m\}$, the following inequality holds with probability at least $1 - \delta$.

\begin{align*}
\mathop{\mathbb{E}}_{P \sim \mathcal{Q}, a_i \sim A(D_i^{:k-1},P)}&\left[\sum_{i=1}^{n}M_{im}(a_i) - \lambda(e-2)V_{im}(a_i)\right]\\
&\leq \frac{1}{\lambda}D_{\mathrm{KL}}((\mathcal{Q}, A_{k-1}^n)||(\mathcal{P}, P^n)) + \frac{1}{\lambda}\mathrm{ln}(1/\delta)
\end{align*}

By rearranging this inequality and applying the cumulative variance bound from Lemma \ref{lem:var_bound}, we have that with probability at least $1 - \delta$:

\begin{align}
\label{eqn:pre_unionb}
\mathop{\mathbb{E}}_{P \sim \mathcal{Q}, a_i \sim A(D_i^{:k-1},P)}\left[\sum_{i=1}^{n}M_{im}(a_i)\right] &\leq \frac{1}{\lambda}D_{\mathrm{KL}}((\mathcal{Q}, A_{k-1}^n)||(\mathcal{P}, P^n))\\
&+ \frac{\lambda(e-2)nm}{b_{\mathrm{min}}} + \frac{1}{\lambda}\mathrm{ln}(1/\delta)\nonumber
\end{align}

Using the union bound, if we replace $\mathrm{ln}(1/\delta)$ with $\mathrm{ln}(m/\delta)$, then Equation (\ref{eqn:pre_unionb}) holds simultaneously for all $k = 1, \dots, m$ with probability at least $1 - \delta$. From the definitions of $\widetilde{\mathcal{R}}(\mathcal{Q})$, $\widehat{\mathcal{R}}(\mathcal{Q})$ and $M_{ij}(a)$:

\begin{align*}
\widehat{\mathcal{R}}(\mathcal{Q}) - \widetilde{\mathcal{R}}(\mathcal{Q}) &= \frac{1}{m}\sum_{k=1}^{m}\frac{1}{nm}\mathop{\mathbb{E}}_{P \sim \mathcal{Q}, a_i \sim A(D_i^{:k-1},P)}\left[\sum_{i=1}^{n}M_{im}(a_i)\right]
\end{align*}

Substituting in the result of Equation (\ref{eqn:pre_unionb}), we have that with probability at least $1 - \delta$:

\begin{align*}
\widehat{\mathcal{R}}(\mathcal{Q}) - \widetilde{\mathcal{R}}(\mathcal{Q}) &\leq \frac{1}{m}\sum_{k=1}^{m}\frac{1}{nm}\frac{1}{\lambda}D_{\mathrm{KL}}((\mathcal{Q}, A_{k-1}^n)||(\mathcal{P}, P^n))\\
&+ \frac{\lambda(e-2)}{b_{\mathrm{min}}} + \frac{1}{nm\lambda}\mathrm{ln}(m/\delta)
\end{align*}

By using Equation (\ref{eqn:kl}) and rearranging this inequality, we obtain:

\begin{align*}
\widetilde{\mathcal{R}}(\mathcal{Q}) &\geq \widehat{\mathcal{R}}(\mathcal{Q}) - \frac{1}{nm\lambda}D_{\mathrm{KL}}(\mathcal{Q}||\mathcal{P}) - \frac{1}{nm^2\lambda}\mathop{\mathbb{E}}_{P \sim \mathcal{Q}}\left[\sum_{i=1}^{n}\sum_{j=1}^{m}D_{\mathrm{KL}}(A(D_i^{:j-1}, P)||P)\right]\\
&- \frac{\lambda(e-2)}{b_{\mathrm{min}}} - \frac{1}{nm\lambda}\mathrm{ln}(m/\delta)
\end{align*}

Finally, the substitution $\lambda^{\prime} = m\lambda$ yields the statement of the lemma. The requirement $\lambda \in [0, b_{\mathrm{min}}]$ then becomes $\lambda^{\prime} \in [0, mb_{\mathrm{min}}]$.

\end{proof}

Next we state our bound on the difference between $\widehat{\mathcal{R}}_{\tau}(\mathcal{Q})$ and $\widetilde{\mathcal{R}}(\mathcal{Q})$. Instead of Bernstein's inequality, this bound uses the Hoeffding-Azuma inequality for supermartingales with bounded differences \cite{azuma1967weighted} \cite{cesa2006prediction}. A proof can be found in the appendix.

\begin{lemma}
\label{lem:azuma}
Let $X_1, \dots, X_n$ be a supermartingale difference sequence (meaning that $\mathbb{E}[X_i|X_1, \dots, X_{i-1}] \leq 0$) such that $X_i \in [a_i, b_i]$ for all $i$ with probability 1. Let $M_n = \sum_{i=1}^{n}X_i$ be the corresponding supermartingale. Then for any $\lambda > 0$:

\begin{equation*}
\mathbb{E}\left[e^{\lambda M_n}\right] \leq e^{\frac{\lambda^2}{8}\sum_{i=1}^{n}(b_i - a_i)^2}
\end{equation*}
\end{lemma}

To utilise Lemma \ref{lem:azuma}, we construct supermartingales with bounded differences from the training data of each task. Define:

\begin{equation}
X_{ij}^{\tau}(a) = \min\left(\frac{1}{b_{ij}(a_{ij})}, 1+\tau\right)\mathbb{I}\{a_{ij}=a\}r_{ij} - R_i(a)\label{eqn:supermax}
\end{equation}

Due to the clipped importance sampling weight, $X_{ij}^{\tau}(a) \in [0, 1+\tau]$ for all $i \leq n$, $j \leq m$ and $a \in \mathcal{A}$. Next, we verify that $\{X_{ij}^{\tau}(a)\}_{j=1}^{m}$ is a supermartingale difference sequence. For any $i \leq n$ and $j \leq m$:

\begin{align*}
\mathbb{E}&\condexp*{X_{ij}^{\tau}(a)}{D_1, \dots, D_{i-1}, D_i^{:j-1}}\\
&= \mathbb{E}\condexp*{\min\left(\frac{1}{b_{ij}(a_{ij})}, 1+\tau\right)\mathbb{I}\{a_{ij}=a\}r_{ij} - R_i(a)}{D_1, \dots, D_{i-1}, D_i^{:j-1}}\\
&\leq \mathbb{E}\condexp*{\frac{1}{b_{ij}(a_{ij})}\mathbb{I}\{a_{ij}=a\}r_{ij} - R_i(a)}{D_1, \dots, D_{i-1}, D_i^{:j-1}}\\
&= 0.
\end{align*}

Now we state our bound on the difference between $\widetilde{\mathcal{R}}(\mathcal{Q})$ and $\widehat{\mathcal{R}}_{\tau}(\mathcal{Q})$ in Lemma \ref{lem:multi_risk_trunc}. The proof of Lemma \ref{lem:multi_risk_trunc} follows the proof of Lemma \ref{lem:multi_risk_bern} except that each application of Bernstein's inequality is replaced with an application of Lemma \ref{lem:azuma}. Therefore, we state the proof in the Appendix.

\begin{lemma}
\label{lem:multi_risk_trunc}
For any hyperprior $\mathcal{P}$, any $\lambda > 0$, any $\tau > 0$ and any $\delta \in (0, 1]$, inequality (\ref{eqn:multi_lower}) holds with probability at least $1 - \delta$ over the training sets $\{D_1, \dots, D_n\}$ and for all hyperposteriors $\mathcal{Q}$

\begin{align}
\widetilde{\mathcal{R}}(\mathcal{Q}) &\geq \widehat{\mathcal{R}}_{\tau}(\mathcal{Q}) -\frac{1}{n\lambda}D_{\mathrm{KL}}(\mathcal{Q}||\mathcal{P}) - \frac{1}{nm\lambda}\mathop{\mathbb{E}}_{P \sim \mathcal{Q}}\left[\sum_{i=1}^{n}\sum_{j=1}^{m}D_{\mathrm{KL}}(A(D_{i}^{:j-1}, P)||P)\right]\nonumber\\
&- \frac{\lambda(1 + \tau)^2}{8m} - \frac{1}{n\lambda}\mathrm{ln}(m/\delta).\label{eqn:multi_lower}
\end{align}
\end{lemma}

Now we are ready to state two lower bounds on the marginal transfer reward, which are our main results. First, we can combine the results from Lemma \ref{lem:transfer_bound} and from Lemma \ref{lem:multi_risk_bern} to obtain our first lower bound.

\begin{theorem}
\label{thm:big_bound_bern}

If condition (\ref{eqn:const_exp_assump}) is satisfied, then for any hyperprior $\mathcal{P}$, any $\lambda_1 > 0$, any $\lambda_2 \in [0, mb_{\mathrm{min}}]$ and any $\delta \in (0, 1]$, inequality (\ref{eqn:main_bern_lower}) holds with probability at least $1 - \delta$ over the tasks and their training sets $(T_1, D_1), \dots, (T_n, D_n)$ and for all hyperposteriors $\mathcal{Q}$

\begin{align}
\label{eqn:main_bern_lower}
\mathcal{R}(\mathcal{Q}) &\geq \widehat{\mathcal{R}}(\mathcal{Q}) - \left(\frac{1}{\lambda_1} + \frac{1}{n\lambda_2}\right)D_{\mathrm{KL}}(\mathcal{Q}||\mathcal{P})\\
&- \frac{1}{nm\lambda_2}\mathop{\mathbb{E}}_{P \sim \mathcal{Q}}\left[\sum_{i=1}^{n}\sum_{j=1}^{m}D_{\mathrm{KL}}(A(D_{i}^{:j-1}, P)||P)\right]\nonumber\\
&- c_n - \frac{\lambda_1}{8n} - \frac{\lambda_2(e - 2)}{b_{\mathrm{min}} m} - \frac{1}{\lambda_1}\mathrm{ln}(2/\delta) - \frac{1}{n\lambda_2}\mathrm{ln}(2m/\delta)\nonumber,
\end{align}

where $c_n$ is the same as in Lemma \ref{lem:transfer_bound}.
\end{theorem}

\begin{proof}
By Lemma \ref{lem:transfer_bound}, for any hyperprior $\mathcal{P}$, $\lambda_1 > 0$ and any $\delta_1 \in (0, 1]$, the following inequality holds with probability at least $1 - \delta_1$

\begin{equation*}
\mathcal{R}(\mathcal{Q}) \geq \widetilde{\mathcal{R}}(\mathcal{Q}) - \frac{1}{\lambda_1}D_{\mathrm{KL}}(\mathcal{Q}||\mathcal{P}) - c_n - \frac{\lambda_1}{8n} - \frac{1}{\lambda_1}\mathrm{ln}(1/\delta_1).
\end{equation*}

By Lemma \ref{lem:multi_risk_bern}, for any hyperprior $\mathcal{P}$, any $\lambda_2 \in [0, mb_{\mathrm{min}}]$ and any $\delta_2 \in (0, 1]$, the following inequality holds with probability at least $1 - \delta_2$

\begin{align*}
\widetilde{\mathcal{R}}(\mathcal{Q}) &\geq \widehat{\mathcal{R}}(\mathcal{Q}) - \frac{1}{n\lambda_2}D_{\mathrm{KL}}(\mathcal{Q}||\mathcal{P}) - \frac{1}{nm\lambda_2}\mathop{\mathbb{E}}_{P \sim \mathcal{Q}}\left[\sum_{i=1}^{n}\sum_{j=1}^{m}D_{\mathrm{KL}}(A(D_i^{:j}, P)||P)\right]\\
&- \frac{\lambda_2(e - 2)}{b_{\mathrm{min}} m} - \frac{1}{n\lambda_2}\mathrm{ln}(m/\delta_2)
\end{align*}

By the union bound, the probability that both inequalities hold simultaneously is at least $1 - \delta_1 - \delta_2$. Therefore, if we set $\delta_1 = \delta_2 = \delta/2$, we have that with probability at least $1 - \delta$

\begin{align*}
\mathcal{R}(\mathcal{Q}) &\geq \widehat{\mathcal{R}}(\mathcal{Q}) - \left(\frac{1}{\lambda_1} + \frac{1}{n\lambda_2}\right)D_{\mathrm{KL}}(\mathcal{Q}||\mathcal{P})\\
&- \frac{1}{nm\lambda_2}\mathop{\mathbb{E}}_{P \sim \mathcal{Q}}\left[\sum_{i=1}^{n}\sum_{j=1}^{m}D_{\mathrm{KL}}(A(D_{i}^{:j-1}, P)||P)\right]\\
&- c_n - \frac{\lambda_1}{8n} - \frac{\lambda_2(e - 2)}{b_{\mathrm{min}} m} - \frac{1}{\lambda_1}\mathrm{ln}(2/\delta) - \frac{1}{n\lambda_2}\mathrm{ln}(2m/\delta).
\end{align*}
\end{proof}

Following the same steps, we can combine the results from Lemma \ref{lem:transfer_bound} and from Lemma \ref{lem:multi_risk_trunc} to obtain a second lower bound.

\begin{theorem}
\label{thm:big_bound_trunc}

If condition (\ref{eqn:const_exp_assump}) is satisfied, then for any hyperprior $\mathcal{P}$, any $\lambda_1 > 0$, any $\lambda_2 > 0$, any $\tau > 0$ and any $\delta \in (0, 1]$, inequality (\ref{eqn:main_trunc_lower}) holds with probability at least $1 - \delta$ over the tasks and their training sets $(T_1, D_1), \dots, (T_n, D_n)$ and for all hyperposteriors $\mathcal{Q}$

\begin{align}
\label{eqn:main_trunc_lower}
\mathcal{R}(\mathcal{Q}) &\geq \widehat{\mathcal{R}}_{\tau}(\mathcal{Q}) - \left(\frac{1}{\lambda_1} + \frac{1}{n\lambda_2}\right)D_{\mathrm{KL}}(\mathcal{Q}||\mathcal{P})\\
&- \frac{1}{nm\lambda_2}\mathop{\mathbb{E}}_{P \sim \mathcal{Q}}\left[\sum_{i=1}^{n}\sum_{j=1}^{m}D_{\mathrm{KL}}(A(D_{i}^{:j-1}, P)||P)\right]\nonumber\\
&- c_n - \frac{\lambda_1}{8n} - \frac{\lambda_2(1+\tau)^2}{8m} - \frac{1}{\lambda_1}\mathrm{ln}(2/\delta) - \frac{1}{n\lambda_2}\mathrm{ln}(2m/\delta)\nonumber,
\end{align}

where $c_n$ is the same as in Lemma \ref{lem:transfer_bound}.
\end{theorem}

The proof of Theorem \ref{thm:big_bound_trunc} can be found in the appendix. The technique of bounding the difference between $\mathcal{R}(\mathcal{Q})$ and $\widehat{\mathcal{R}}_{\tau}(\mathcal{Q})$ by adding bounds on the difference between $\mathcal{R}(\mathcal{Q})$ and $\widetilde{\mathcal{R}}(\mathcal{Q})$ and the difference between $\widetilde{\mathcal{R}}(\mathcal{Q})$ and $\widehat{\mathcal{R}}_{\tau}(\mathcal{Q})$ is borrowed from \cite{pentina2014pac}. However, the two bounds being added together are both novel.

The lower bounds in Theorem \ref{thm:big_bound_bern} and Theorem \ref{thm:big_bound_trunc} contain two complexity penalties. The first is the KL divergence between the hyperposterior and the hyperprior. This penalty is small when the hyperposterior is close to the hyperprior. The second is the expected average KL divergence between the posteriors returned by the base learning algorithm and the priors. This penalty is small when the hyperposterior assigns high probability density to priors that are close to the posteriors returned by the base learning algorithm.

The asymptotic behaviour of the bounds depends on the values of $\lambda_1$ and $\lambda_2$. If the assumption in Equation (\ref{eqn:cond_exp_assump}) is satisfied, then $c_n$ decays to $0$ as $n \to \infty$ with high probability. If we ignore this term, then the difference between the marginal transfer reward and the (clipped) empirical multi-task reward is of order $\mathcal{O}(\frac{1}{\lambda_1} + \frac{1}{n\lambda_2} + \frac{1}{\lambda_2} + \frac{\lambda_1}{n} + \frac{\lambda_2}{m} + \frac{\mathrm{ln}(m)}{n\lambda_2})$. If we set $\lambda_1 = \mathcal{O}(\sqrt{n})$ and $\lambda_2 = \mathcal{O}(\sqrt{m})$, then the difference is of order $\mathcal{O}(\frac{1}{\sqrt{n}} + \frac{1}{n\sqrt{m}} + \frac{1}{\sqrt{m}} + \frac{\mathrm{ln}(m)}{n\sqrt{m}})$. Hence, for this choice of $\lambda_1$ and $\lambda_2$, the difference between the marginal transfer reward and the empirical multi-task reward decays to $0$ as $n \to \infty$ and $m \to \infty$.

\section{Algorithms For Lifelong Multi-Armed Bandits}
\label{sec:alg}

We propose lifelong learning algorithms that use the lower bounds in Theorem \ref{thm:big_bound_bern} and Theorem \ref{thm:big_bound_trunc} as their objective functions. These algorithms use either variational inference (VI) or Markov chain Monte Carlo (MCMC) to approximate the hyperposterior that maximises the lower bound. First, we will describe the characteristics that the algorithms all share.

Each prior $P$ and posterior $Q$ is a probability vector with $K = |\mathcal{A}|$ elements. The $a$th element is the probability of selecting action $a$. We restrict priors to be the softmax of a weight vector $\bs{w} \in \mathbb{R}^{K}$. This means that the hyperprior and hyperposterior can be any distribution over $\mathbb{R}^{K}$. Let $P_{\bs{w}}$ denote the prior probability vector that is softmax of $\bs{w}$. The choice of base learning algorithm is somewhat arbitrary. The only requirements are that it should take a prior $P$ over actions and a dataset $D$ as inputs and return a posterior $Q$ over actions. Since we will need to compute the posterior returned by the base learning algorithm many times when evaluating the lower bound, it is preferable for the posterior returned by the base learning algorithm to have a closed-form solution. With these points in mind, we use the following base learning algorithm:

\begin{equation*}
A(D_i, P) = \argmax_{Q}\left\{\sum_{a \in \mathcal{A}}Q(a)\frac{1}{|I_a|}\sum_{j \in I_a}r_{ij} - \frac{K}{\sqrt{m}}D_{\mathrm{KL}}(Q||P)\right\}.
\end{equation*}

$I_a$ is the set of indices where $a_{ij} = a$. The posterior returned by this base learning algorithm has the following closed-form solution:

\begin{equation*}
Q(a) = \frac{P(a)\mathrm{exp}\left(\frac{\sqrt{m}}{K|I_a|}\sum_{j \in I_a}r_{ij}\right)}{\sum_{a^{\prime} \in \mathcal{A}}P(a^{\prime})\mathrm{exp}\left(\frac{\sqrt{m}}{K|I_a^{\prime}|}\sum_{j \in I_a^{\prime}}r_{ij}\right)}
\end{equation*}

Each behaviour policy can be set equal to the current posterior $Q$. Since the Bernstein bound in Theorem \ref{thm:big_bound_bern} depends on $b_{\mathrm{min}} = \min_{i\leq n, j\leq m, a \in \mathcal{A}}(b_{ij}(a))$, we instead set the behaviour policy to be an $\epsilon$-soft version of the current posterior when using this bound. That is, with probability $1 - \epsilon$, actions are sampled from $Q$, and with probability $\epsilon$, actions are sampled uniformly at random. This guarantees that $b_{\mathrm{min}} \geq \epsilon/K$, where $K$ is the number of actions. Therefore, $b_{\mathrm{min}}$ in Theorem \ref{thm:big_bound_bern} can be replaced with $\epsilon/K$.

We set $\lambda_1 = T_1\sqrt{n}$ and $\lambda_2 = T_2\sqrt{m}$, where $T_1$ and $T_2$ are positive temperature parameters. Since $\lambda_1$ and $\lambda_2$ must be chosen in advance and cannot depend on any observed data for our bounds to be valid, $T_1$ and $T_2$ must also be chosen in advance. Optionally, one can use the bound value to inform the choice of the temperature parameters. First, a grid of (say $N$) $T_1$ and $T_2$ values can be constructed. Then, using a union bound argument, if we replace $\delta$ with $\delta/N$, then either of our lower bounds holds simultaneously for each $T_1$ and $T_2$ with probability greater than $1 - \delta$. Then one can use the lower bound with whichever pair of $T_1$ and $T_2$ maximises its value as an objective function. This approach could result in nearly the best possible bound value, but will not necessarily result in the best reward obtained. In our experiments, we run our algorithms several times with different fixed values of $T_1$ and $T_2$, and we find that the values of $T_1$ and $T_2$ that give the best bound values do not give the best performance in terms of reward.

Both our proposed lifelong learning algorithms approximate the hyperposterior that maximises one of our lower bounds. If we are using the Bernstein bound in Theorem \ref{thm:big_bound_bern}, this hyperposterior is the solution of the following optimisation problem:

\begin{align}
\argmax_{\mathcal{Q}}\bigg\{&\widehat{\mathcal{R}}(\mathcal{Q}) - \frac{1}{T_2nm\sqrt{m}}\mathop{\mathbb{E}}_{\bs{w} \sim \mathcal{Q}}\left[\sum_{i=1}^{n}\sum_{j=1}^{m}D_{\mathrm{KL}}(A(D_{i}^{:j-1}, P_{\bs{w}})||P_{\bs{w}})\right]\nonumber\\
&- \frac{T_1\sqrt{n} + T_2n\sqrt{m}}{T_1T_2n\sqrt{nm}}D_{\mathrm{KL}}(\mathcal{Q}||\mathcal{P})\bigg\}.\label{eqn:opt_hyperpost_lambda}
\end{align}

\subsection{Variational Inference}

Here, we describe an algorithm that approximates the solution of the optimisation problem in Equation (\ref{eqn:opt_hyperpost_lambda}) with variational inference techniques. We describe the algorithm using the Bernstein bound as an example, but the method works in the same way with the clipping bound from Theorem \ref{thm:big_bound_trunc}. We instead solve a slightly different optimisation problem:

\begin{align}
\argmax_{\theta}\bigg\{&\widehat{\mathcal{R}}(\mathcal{Q}_{\theta}) - \frac{1}{T_2nm\sqrt{m}}\mathop{\mathbb{E}}_{\bs{w} \sim \mathcal{Q}_{\theta}}\left[\sum_{i=1}^{n}\sum_{j=1}^{m}D_{\mathrm{KL}}(A(D_{i}^{:j-1}, P_{\bs{w}})||P_{\bs{w}})\right]\nonumber\\
&- \frac{T_1\sqrt{n} + T_2n\sqrt{m}}{T_1T_2n\sqrt{nm}}D_{\mathrm{KL}}(\mathcal{Q}_{\theta}||\mathcal{P})\bigg\}.\label{eqn:vi_opt}
\end{align}

$\mathcal{Q}_{\theta}$ is a member of a parametric family of distributions with parameter $\theta$. We choose $\mathcal{Q}_{\theta} = \mathcal{N}(\bs{\mu}_{\mathcal{Q}}, \bs{\sigma}_{\mathcal{Q}}I)$, so $\theta = (\bs{\mu}_{\mathcal{Q}}, \bs{\sigma}_{\mathcal{Q}})$, and $\mathcal{P} = \mathcal{N}(\bs{\mu}_{\mathcal{P}}, \bs{\sigma}_{\mathcal{P}}I)$. Our goal is now to find the optimal $\theta$.

\begin{algorithm}[H]
\caption{PAC-Bayes VI}
\label{alg:pac_bayes_vi}
\begin{algorithmic}
\STATE {\bfseries Input:} Task distribution $\mathcal{T}$, base learning algorithm $A$, hyperprior $\mathcal{P}$, temperature parameters $T_1, T_2$
\STATE Initialise hyperposterior $\mathcal{Q}_{\theta} \gets \mathcal{P}$
\FOR{$i=1$ {\bfseries to} $n$}
\STATE Sample task $T_i \sim \mathcal{T}$
\STATE Sample prior $P_{\bs{w}} \sim \mathcal{Q}_{\theta}$
\STATE Initialise task posterior $Q \gets P_{\bs{w}}$
\STATE Initialise task dataset $D_i \gets ()$
\FOR{$j=1$ {\bfseries to} $m$}
\STATE Get behaviour policy $b_{ij} \gets B(Q, \epsilon)$
\STATE Sample $a_{ij} \sim b_{ij}(a)$ and $r_{ij} \sim \rho_{i}(r|a_{ij})$
\STATE Append $(a_{ij}, r_{ij})$ to $D_i$
\STATE Update task posterior $Q \gets A(D_i, P_{\bs{w}})$
\ENDFOR
\FOR{$k$ iterations}
\STATE $\theta \gets \theta + \eta\pp{}{}{\theta}L_{\mathrm{VI}}(\theta)$ (or Adam)
\ENDFOR
\ENDFOR
\end{algorithmic}
\end{algorithm}

We use Monte Carlo approximations of the expected values in Equation (\ref{eqn:vi_opt}). In particular,

\begin{align*}
\widehat{\mathcal{R}}(\mathcal{Q}_{\theta}) &= \mathop{\mathbb{E}}_{\bs{w} \sim \mathcal{Q}_{\theta}}\left[\frac{1}{n}\sum_{i=1}^{n}\frac{1}{m}\sum_{j=1}^{m}\widehat{R}_i(A(D_i^{:j-1}, P_{\bs{w}}), D_i)\right]\\
&\approx \frac{1}{n}\sum_{i=1}^{n}\frac{1}{m}\sum_{j=1}^{m}\widehat{R}_i(A(D_i^{:j-1}, P_{\bs{w}_{ij}}), D_i),
\end{align*}

and

\begin{equation*}
\mathop{\mathbb{E}}_{\bs{w} \sim \mathcal{Q}_{\theta}}\left[\sum_{i=1}^{n}\sum_{j=1}^{m}D_{\mathrm{KL}}(A(D_{i}^{:j-1}, P_{\bs{w}})||P_{\bs{w}})\right] \approx \sum_{i=1}^{n}\sum_{j=1}^{m}D_{\mathrm{KL}}(A(D_{i}^{:j-1}, P_{\bs{w}_{ij}})||P_{\bs{w}_{ij}}),
\end{equation*}

where each $\bs{w}_{ij}$ is an independent sample from $\mathcal{Q}_{\theta}$. With these Monte Carlo approximations, the objective function used in the algorithm is:

\begin{align*}
L_{\mathrm{VI}}(\theta) &= \frac{1}{n}\sum_{i=1}^{n}\frac{1}{m}\sum_{j=1}^{m}\widehat{R}_i(A(D_i^{:j-1}, P_{\bs{w}_{ij}}), D_i)\\
&- \frac{1}{T_2nm\sqrt{m}}\sum_{i=1}^{n}\sum_{j=1}^{m}D_{\mathrm{KL}}(A(D_{i}^{:j-1}, P_{\bs{w}_{ij}})||P_{\bs{w}_{ij}})\\
&- \frac{T_1\sqrt{n} + T_2n\sqrt{m}}{T_1T_2n\sqrt{nm}}D_{\mathrm{KL}}(\mathcal{Q}_{\theta}||\mathcal{P})
\end{align*}

We use Adam \cite{kingma2014adam} to maximise the lower bound with respect to $\theta$. Gradients of the Monte Carlo approximations with respect to $\theta$ are computed using the reparameterisation trick \cite{kingma2013auto}. We call this algorithm PAC-Bayes VI. Algorithm \ref{alg:pac_bayes_vi} provides pseudocode for PAC-Bayes VI. In Algorithm \ref{alg:pac_bayes_vi}, $B(Q, \epsilon)$ denotes the $\epsilon$-soft version of $Q$. Note that if the clipping bound is used, then there is no need to use an $\epsilon$-soft version of $Q$ as a behaviour policy.

Since $D_{\mathrm{KL}}(\mathcal{Q}_{\theta}||\mathcal{P})$ is available in closed-form, the value of the lower bound at $\mathcal{Q}_{\theta}$ can be easily computed if we assume that $c_n = 0$. 

\subsection{Markov Chain Monte Carlo}

Here, we describe an algorithm that uses Markov Chain Monte Carlo (MCMC) techniques to draw samples from the hyperposterior that is solution to the optimisation problem in Equation (\ref{eqn:opt_hyperpost_lambda}). Firstly, define

\begin{align*}
\phi(\bs{w}) &= \frac{T_1T_2n\sqrt{nm}}{T_1\sqrt{n} + T_2n\sqrt{m}}\frac{1}{n}\sum_{i=1}^{n}\frac{1}{m}\sum_{j=1}^{m}\widehat{R}_i(A(D_i^{:j-1}, P_{\bs{w}}), D_i)\\
&- \frac{T_1T_2n\sqrt{nm}}{T_1\sqrt{n} + T_2n\sqrt{m}}\frac{1}{T_2nm\sqrt{m}}\sum_{i=1}^{n}\sum_{j=1}^{m}D_{\mathrm{KL}}(A(D_{i}^{:j}, P_{\bs{w}})||P_{\bs{w}}).
\end{align*}

We have that

\begin{align*}
\mathop{\mathbb{E}}_{\bs{w} \sim \mathcal{Q}}\left[\phi(\bs{w})\right] &= \frac{T_1T_2n\sqrt{nm}}{T_1\sqrt{n} + T_2n\sqrt{m}}\widehat{\mathcal{R}}(\mathcal{Q})\\
&- \frac{T_1\sqrt{n}}{m(T_1\sqrt{n} + T_2n\sqrt{m})}\mathop{\mathbb{E}}_{\bs{w} \sim \mathcal{Q}}\left[\sum_{i=1}^{n}\sum_{j=1}^{m}D_{\mathrm{KL}}(A(D_{i}^{:j-1}, P_{\bs{w}})||P_{\bs{w}})\right].
\end{align*}

The optimisation problem in Equation (\ref{eqn:opt_hyperpost_lambda}) can be rewritten as

\begin{equation*}
\argmax_{\mathcal{Q}}\left\{\mathop{\mathbb{E}}_{\bs{w} \sim \mathcal{Q}}\left[\phi(\bs{w})\right] - D_{\mathrm{KL}}(\mathcal{Q}||\mathcal{P})\right\}
\end{equation*}

This type of problem appears frequently in the PAC-Bayesian literature and has previously been studied by \cite{catoni2004statistical}, \cite{guedj2019primer}. It is known that the Gibbs distribution is the solution. In this case, the Gibbs distribution has the probability density function

\begin{equation*}
\mathcal{Q}(\bs{w}) \propto \mathcal{P}(\bs{w})e^{\phi(\bs{w})}.
\end{equation*}

Furthermore, the maximum value attained by the Gibbs distribution is:

\begin{equation}
\label{eqn:gibb_max}
\max_{\mathcal{Q}}\left\{\mathop{\mathbb{E}}_{\bs{w} \sim \mathcal{Q}}\left[\phi(\bs{w})\right] - D_{\mathrm{KL}}(\mathcal{Q}||\mathcal{P})\right\} = \mathrm{ln}\left(\mathop{\mathbb{E}}_{\bs{w} \sim \mathcal{P}}\left[e^{\phi(\bs{w})}\right]\right)
\end{equation}

Though we can only compute the unnormalised density function of the Gibbs distribution, we can still sample it using MCMC. We choose $\mathcal{P} = \mathcal{N}(\bs{\mu}_{\mathcal{P}}, \bs{\sigma}_{\mathcal{P}}I)$ and we use Preconditioned Stochastic Gradient Langevin Dynamics (pSGLD) \cite{li2016preconditioned} to draw samples from $\mathcal{P}(\bs{w})e^{\phi(\bs{w})}$. We call this algorithm PAC-Bayes MCMC. Algorithm \ref{alg:pac_bayes_mcmc} provides pseudocode for PAC-Bayes MCMC.

This time, $D_{\mathrm{KL}}(\mathcal{Q}||\mathcal{P})$ is not available in closed-form. However, we can use equation (\ref{eqn:gibb_max}) to rewrite the Bernstein (or clipping) lower bound as:

\begin{align*}
&\frac{T_1\sqrt{n} + T_2n\sqrt{m}}{T_1T_2n\sqrt{nm}}\mathrm{ln}\left(\mathop{\mathbb{E}}_{\bs{w} \sim \mathcal{P}}\left[e^{\phi(\bs{w})}\right]\right) - c_n - \frac{T_1}{8\sqrt{n}}\\
&- \frac{T_2(e - 2)}{b_{\mathrm{min}} \sqrt{m}} - \frac{1}{T_1\sqrt{n}}\mathrm{ln}(2/\delta) + \frac{1}{T_2n\sqrt{m}}\mathrm{ln}(2m/\delta)
\end{align*}

\begin{algorithm}[H]
\caption{PAC-Bayes MCMC}
\label{alg:pac_bayes_mcmc}
\begin{algorithmic}
\STATE {\bfseries Input:} Task distribution $\mathcal{T}$, base learning algorithm $A$, hyperprior $\mathcal{P}$, temperature parameters $T_1, T_2$
\STATE Initialise prior weight vector $\bs{w} \sim \mathcal{P}$
\FOR{$i=1$ {\bfseries to} $n$}
\STATE Sample task $T_i \sim \mathcal{T}$
\STATE Set prior $P \gets P_{\bs{w}}$
\STATE Initialise task posterior $Q \gets P$
\STATE Initialise task dataset $D_i \gets ()$
\FOR{$j=1$ {\bfseries to} $m$}
\STATE Get behaviour policy $b_{ij} \gets B(Q, \epsilon)$
\STATE Sample $a_{ij} \sim b_{ij}(a)$ and $r_{ij} \sim \rho_{i}(r|a_{ij})$
\STATE Append $(a_{ij}, r_{ij})$ to $D_i$
\STATE Update task posterior $Q \gets A(D_i, P)$
\ENDFOR
\FOR{$k$ iterations}
\STATE $\bs{w} \gets \texttt{pSGLD\_step}(\bs{w}, \mathcal{P}(\bs{w})e^{\phi(\bs{w})})$
\ENDFOR
\ENDFOR
\end{algorithmic}
\end{algorithm}

All that remains is to calculate or approximate $\mathrm{ln}\left(\mathop{\mathbb{E}}_{\bs{w} \sim \mathcal{P}}\left[e^{\phi(\bs{w})}\right]\right)$. We can apply Jensen's inequality to obtain a lower bound that can easily be approximated with a standard Monte Carlo estimate:

\begin{equation*}
\mathrm{ln}\left(\mathop{\mathbb{E}}_{\bs{w} \sim \mathcal{P}}\left[e^{\phi(\bs{w})}\right]\right) \geq \mathop{\mathbb{E}}_{\bs{w} \sim \mathcal{P}}\left[\mathrm{ln}\left(e^{\phi(\bs{w})}\right)\right]
\end{equation*}

As shown in \cite{burda2016importance}, if we replace $e^{\phi(\bs{w})}$ with an average of several samples $\frac{1}{K}\sum_{k=1}^{K}e^{\phi(\bs{w}_k)}$ from $\mathcal{P}$, then the lower bound on $\mathrm{ln}\left(\mathbb{E}_{\bs{w} \sim \mathcal{P}}\left[e^{\phi(\bs{w})}\right]\right)$ approaches the true value as $K$ goes to infinity:

\begin{equation*}
\lim_{K \to \infty}\mathop{\mathbb{E}}_{\bs{w}_1, \dots, \bs{w}_K \sim \mathcal{P}^K}\left[\mathrm{ln}\left(\frac{1}{K}\sum_{k=1}^{K}e^{\phi(\bs{w}_k)}\right)\right] = \mathrm{ln}\left(\mathop{\mathbb{E}}_{\bs{w} \sim \mathcal{P}}\left[e^{\phi(\bs{w})}\right]\right)
\end{equation*}

\section{Experiments}

We tested our proposed lifelong learning algorithms in three environments. In each environment, tasks are multi-armed bandit problems with 10 or 20 actions and binary rewards. The reward distribution for action $a$ is a Bernoulli distribution with parameter $p_a$. The distribution over each $p_a$ is a Beta distribution with parameters $\alpha_a$ and $\beta_a$. Therefore, sampling a new task from the environment means sampling the parameter $p_a$ for each arm from a Beta distribution with parameters $\alpha_a$ and $\beta_a$.

In environment one, the shape parameters of the Beta distributions are as follows:

\begin{equation*}
(\alpha_i, \beta_i) = \left\{\begin{array}{ll} (5.0, 20.0) & \text{if } i \in \{0, 1, \dots, 7\}\\ (20.0, 5.0) & \text{if } i \in \{8, 9\} \end{array}\right..
\end{equation*}

When sampling tasks from this environment, the expected values of $p_0, \dots, p_7$ are all 0.2 and the expected values of $p_8$ and $p_9$ are both 0.8. $p_a$ has small variance for every $a$. Each task sampled from this environment will be very similar, and in almost every task sampled from this environment, there will be $2$ good actions and $8$ bad actions. A good prior for tasks sampled from this environment would assign high probability to actions $8$ and $9$, reducing the problem of choosing between 10 actions to choosing between the two good actions. The second environment is the same as the first environment, except with 20 actions instead of 10 actions. The proportions of good and bad actions is still the same. In environment two, the shape parameters of the Beta distributions are:

\begin{equation*}
(\alpha_i, \beta_i) = \left\{\begin{array}{ll} (5, 20) & \text{if } i \in \{0, 1, \dots, 15\}\\ (20, 5) & \text{if } i \in \{16, 17, 18, 19\} \end{array}\right..
\end{equation*}

The third environment also has 20 actions. The shape parameters of the Beta distributions are:

\begin{equation*}
(\alpha_i, \beta_i) = \left\{\begin{array}{ll} (1, 4) & \text{if } i \in \{0, 1, \dots, 9\}\\
(5, 20) & \text{if } i = 10\\
(\tfrac{20}{3}, \tfrac{55}{3}) & \text{if } i = 11\\
(\tfrac{25}{3}, \tfrac{50}{3}) & \text{if } i = 12\\
\qquad \vdots & \-\\\
(20, 5) & \text{if } i = 19\\
\end{array}\right..
\end{equation*}

\begin{figure}[H]
\centering
\includegraphics[width=0.9\textwidth]{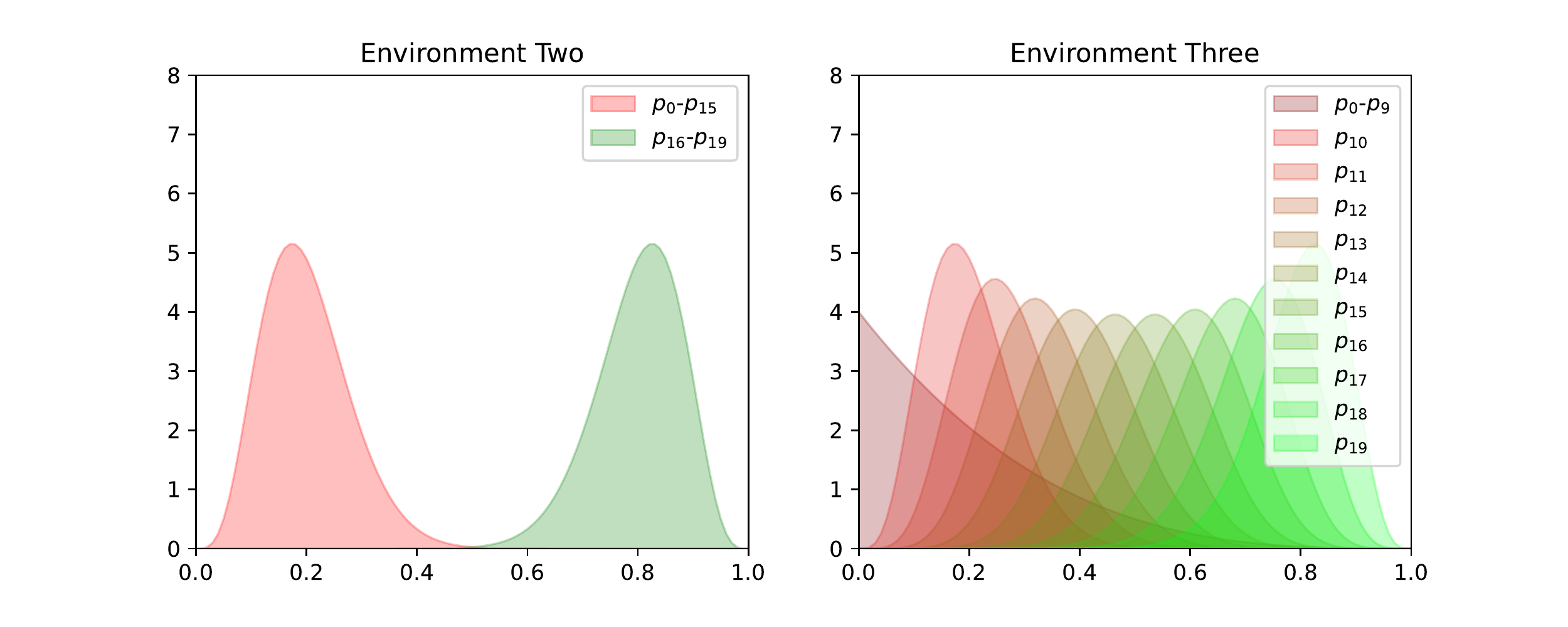}
\caption{The probability density of $p_a$ for each action in environment two (left) and environment three (right).}
\label{fig:envs}
\end{figure}

In environment three, $p_a$ has large variance for the first 10 actions and small variance for the last 10 actions. The expected value of $p_a$ is 0.2 for the first 10 actions and increases linearly from 0.2 to 0.8 for actions 10 to 19. In environment three, it is more difficult to identify which actions tend to have the best rewards and therefore harder to learn a useful prior. Figure \ref{fig:envs} shows the probability density of $p_a$ for each action in environments two and three. The densities for environment one are the same as those of environment two, except with only 10 total actions.

We compared our PAC-Bayes VI and PAC-Bayes MCMC algorithms with both bounds against two benchmark methods. The first benchmark uses the base learning algorithm with a uniform prior for every task. We call this benchmark Learning From Scratch (LFS). The second benchmark is called Adaptive Ridge Regression (ARR) \cite{pentina2014pac}. It uses the base learning algorithm for every task, initially with a uniform prior. After every task, the prior is set to be equal to the average of the posterior distributions learned on all previous tasks. If $p_{n+1}(a)$ is the $a$th element of the prior probability vector for task $n+1$ and $q_i(a)$ is the $a$th element of the posterior probability vector from task $i$, then

\begin{equation*}
p_{n+1}(a) = \frac{1}{n}\sum_{i=1}^{n}q_{i}(a)
\end{equation*}

In each experiment, $m=20$ action-reward pairs were sampled from each task and the number of tasks ranged from $n = 1$ to $100$. We tested our algorithms with temperature constants $(T_1, T_2) \in \{(5, 1), (15, 3),$ $(50, 10)\}$ and reported results with each. For the Bernstein bound, we additionally tested $T_1 = 5$ and $T_2 = (\epsilon/K)\sqrt{m}$, which is the largest possible value for $T_2$ such that the bound is valid. We recorded the average of the $m$ rewards obtained in each task as well as the value of the each lower bound. The reported bound values assume that $c_n = 0$. Returning to our second assumption in Sect.\ \ref{simple}, this means we have assumed that $n$ is large enough for $c_n$ to have decayed to $0$. Each run of the experiment was an average of $10$ runs. We repeated this 50 times and we report the mean and standard deviation of the average reward.

\begin{figure}[H]
\centering
\includegraphics[width=0.73\textwidth]{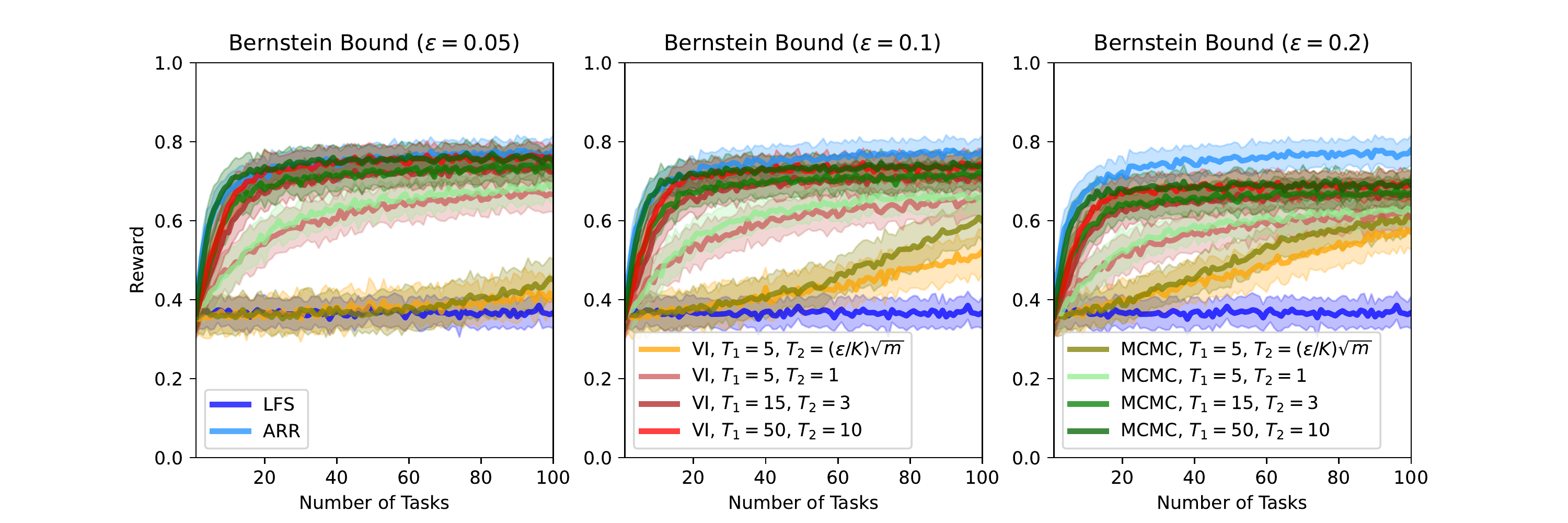}

\includegraphics[width=0.73\textwidth]{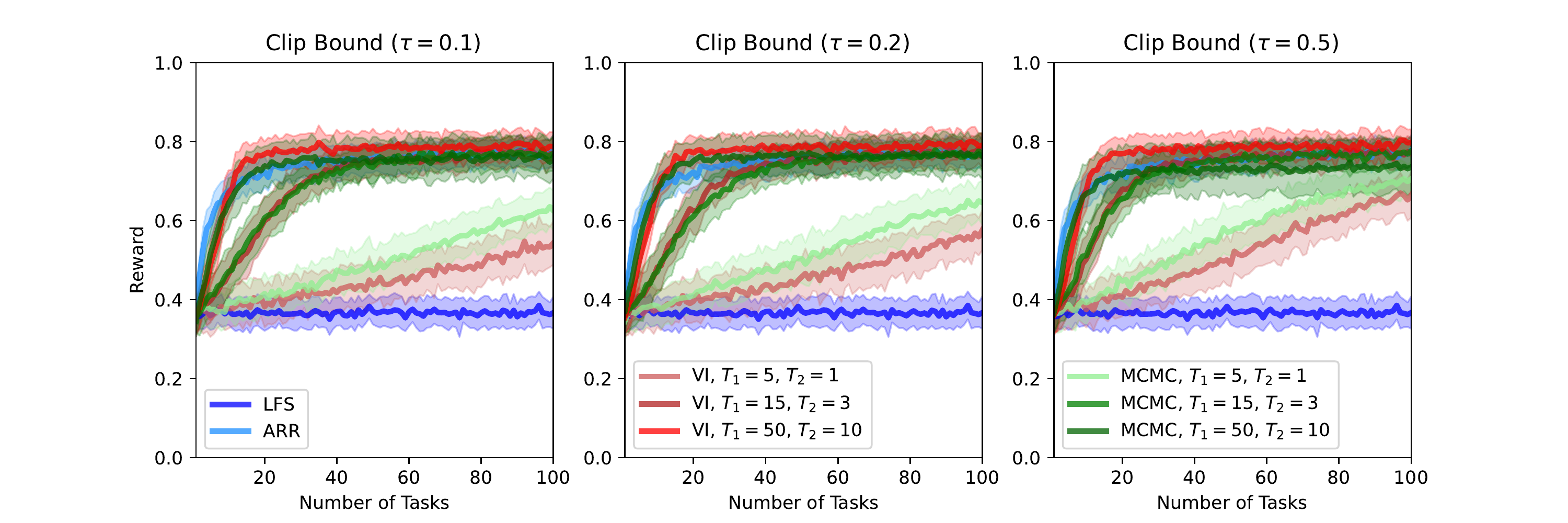}

\includegraphics[width=0.73\textwidth]{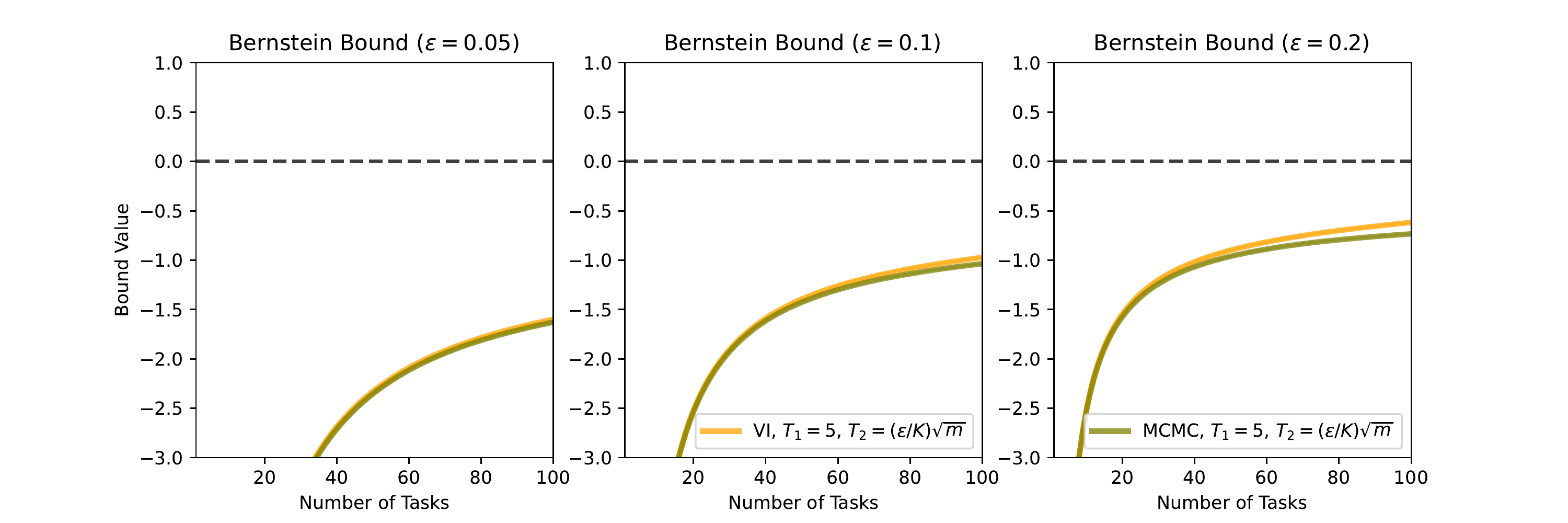}

\includegraphics[width=0.73\textwidth]{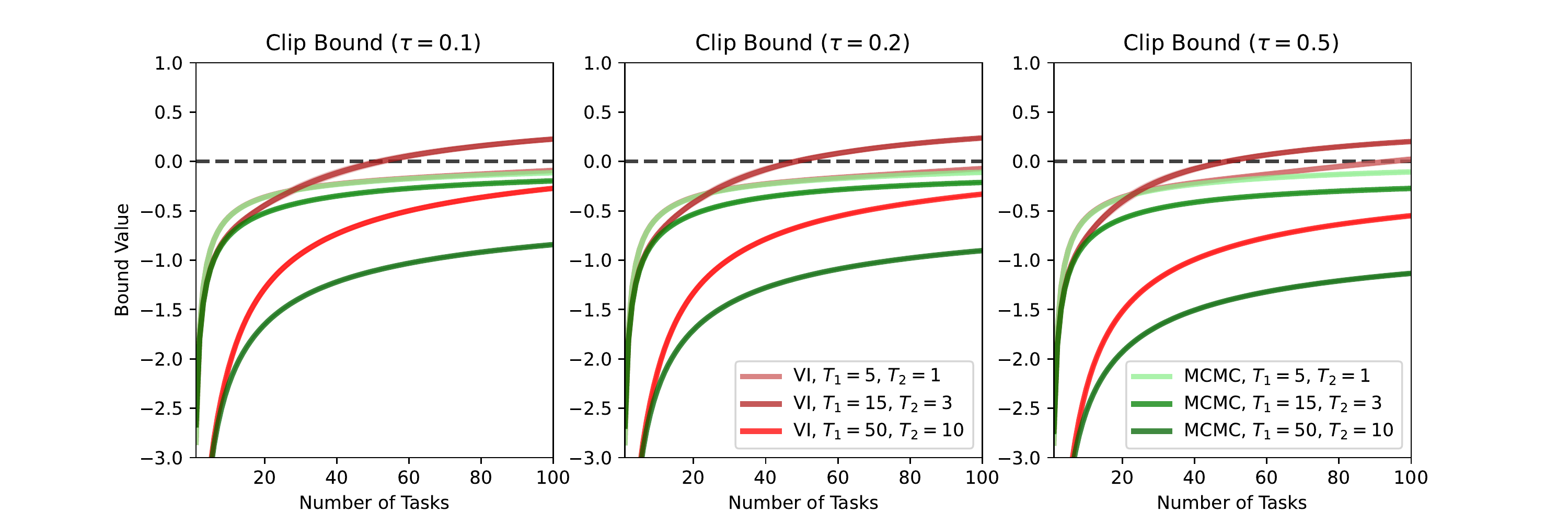}
\caption{The average reward obtained with the Bernstein bound (top row) and with the clipping bound (upper middle row) in tasks sampled from environment one. The values of the Bernstein bound (lower middle row) and the clipping bound (bottom row) are also shown. The solid lines show the mean average reward and the shaded regions show the mean $\pm$ 1 standard deviation.}
\label{fig:a10}
\end{figure}

First, we investigate the performance of our algorithms when the hyperprior is uninformative. Recall from the beginning of Sect.\ \ref{sec:alg}, that the hyperprior is a distribution over the weight vector of a softmax policy. For our experiments with an uninformative hyperprior, we use a standard Gaussian distribution as the hyperprior. Figure \ref{fig:a10} shows the average reward obtained with the baselines and with PAC-Bayes VI and PAC-Bayes MCMC in environment one. The values of the Bernstein and clipping bounds are also shown.

\begin{figure}[H]
\centering
\includegraphics[width=0.73\textwidth]{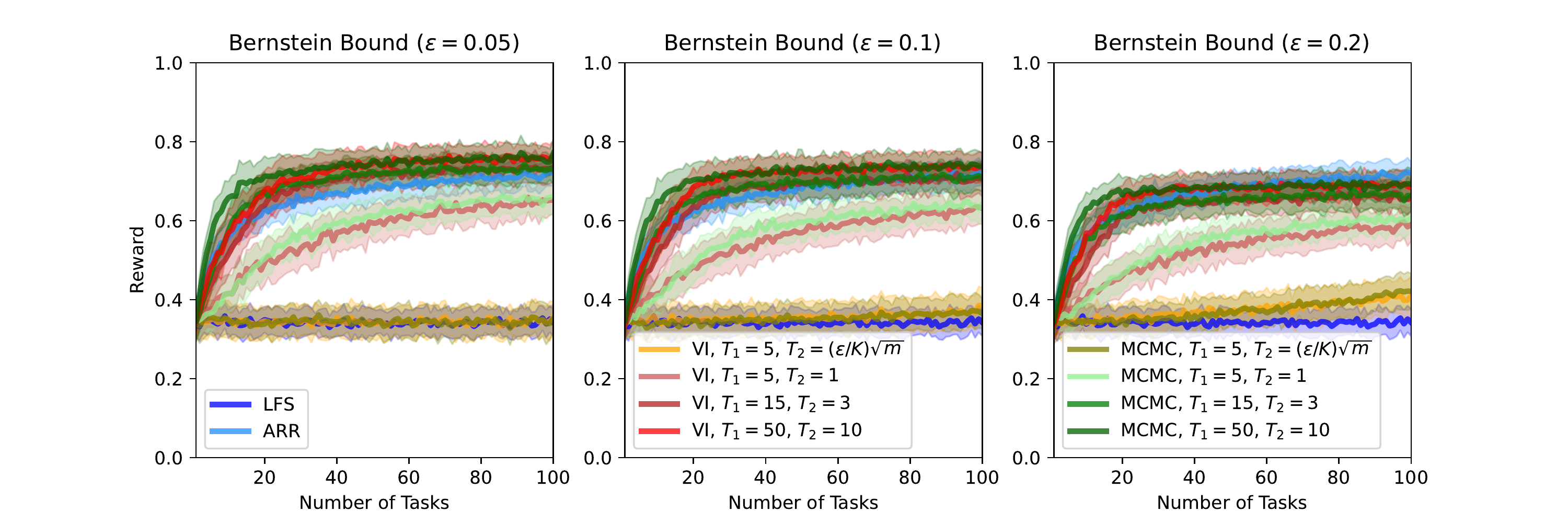}

\includegraphics[width=0.73\textwidth]{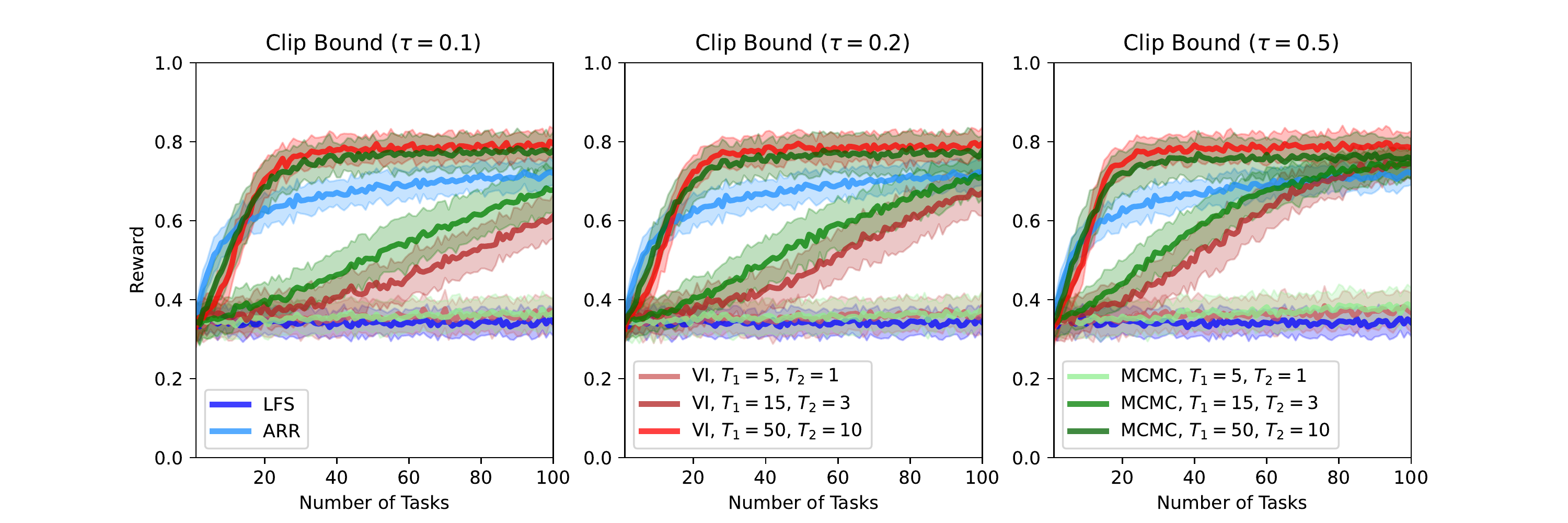}

\includegraphics[width=0.73\textwidth]{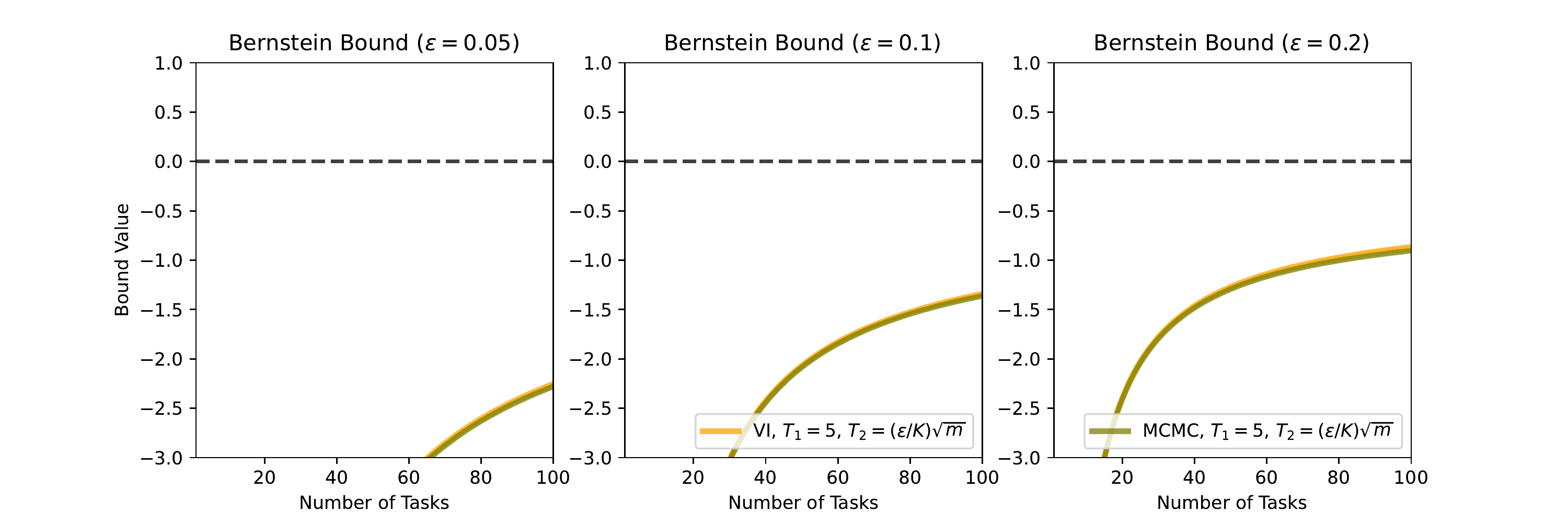}

\includegraphics[width=0.73\textwidth]{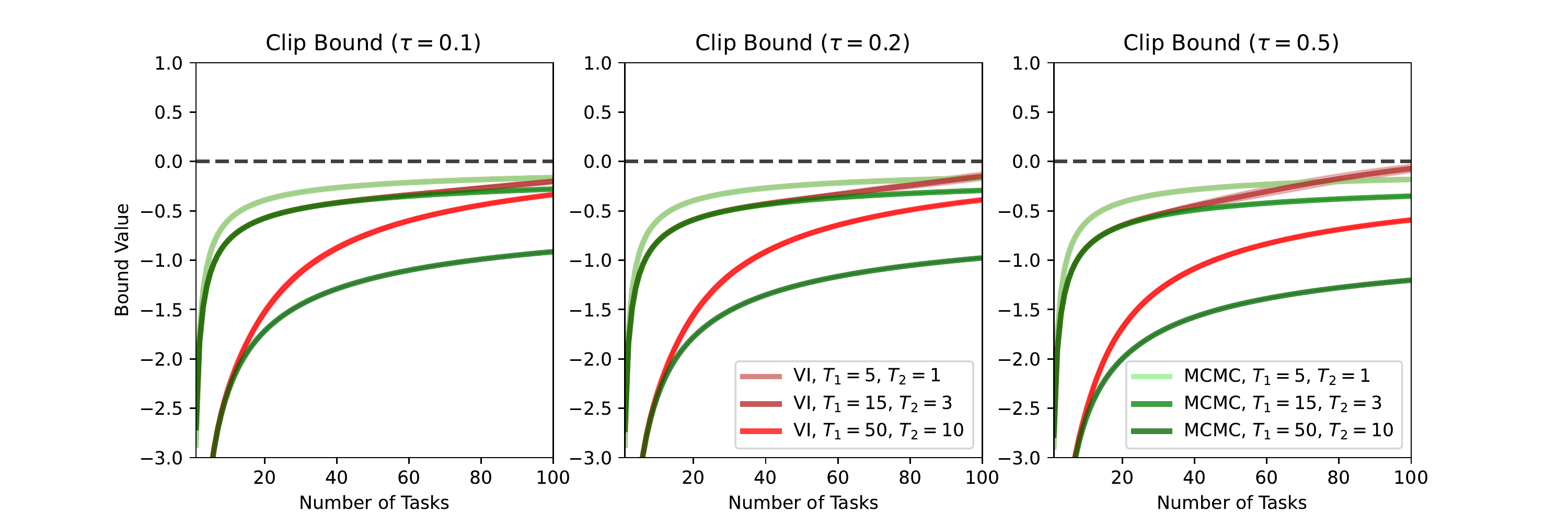}
\caption{The average reward obtained with the Bernstein bound (top row) and with the clipping bound (upper middle row) in tasks sampled from environment two. The values of the Bernstein bound (lower middle row) and the clipping bound (bottom row) are also shown. The solid lines show the mean average reward and the shaded regions show the mean $\pm$ 1 standard deviation.}
\label{fig:a20}
\end{figure}

In Figure \ref{fig:a10}, ARR reached an average reward of 0.779 by task 100. When the Bernstein bound was used and $\epsilon$ was 0.05, PAC-Bayes VI and PAC-Bayes MCMC with $T_1=50$ and $T_2=10$ reached only 0.753 and 0.749, respectively. When $\epsilon$ was 0.1 or 0.2, the average reward for both PAC-Bayes algorithms plateaued at a lower average reward. When the clipping bound was used, PAC-Bayes VI with the highest temperatures and $\tau$ equal to $0.1$, $0.2$ and $0.5$ reached average rewards of 0.788, 0.791 and 0.798 respectively, which is slightly higher than that of ARR. At the highest temperatures, PAC-Bayes MCMC was slightly worse, reaching only 0.762, 0.763 and 0.736 respectively for each value of $\tau$.

\begin{figure}[H]
\centering
\includegraphics[width=0.73\textwidth]{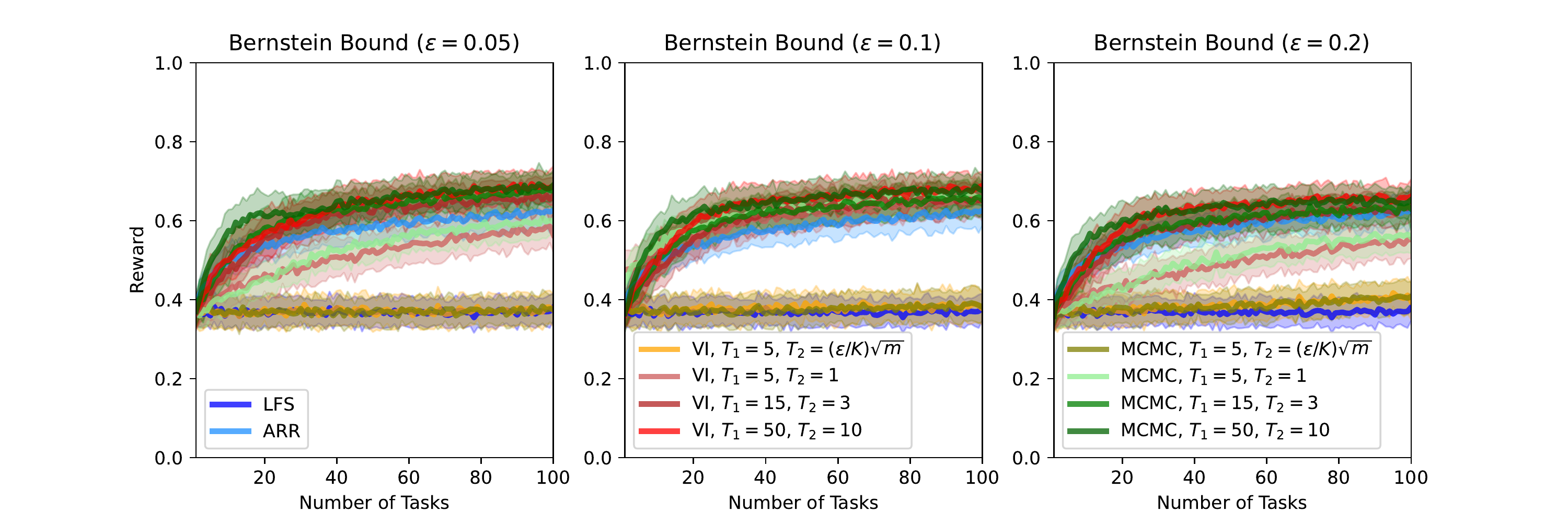}

\includegraphics[width=0.73\textwidth]{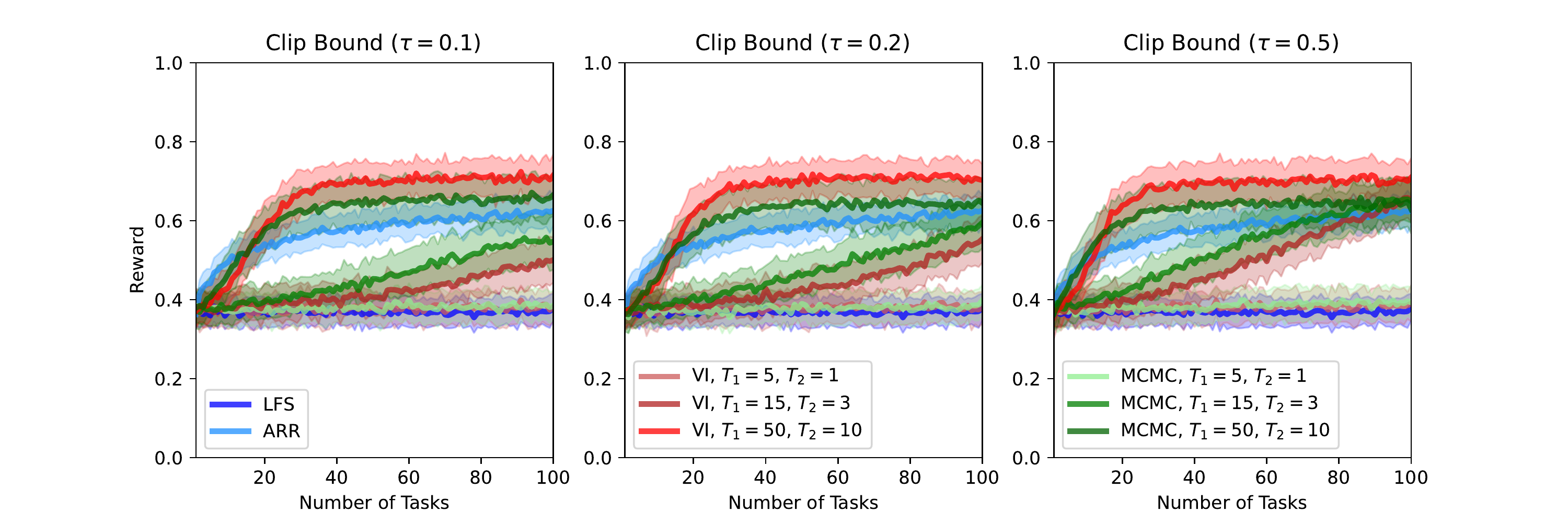}

\includegraphics[width=0.73\textwidth]{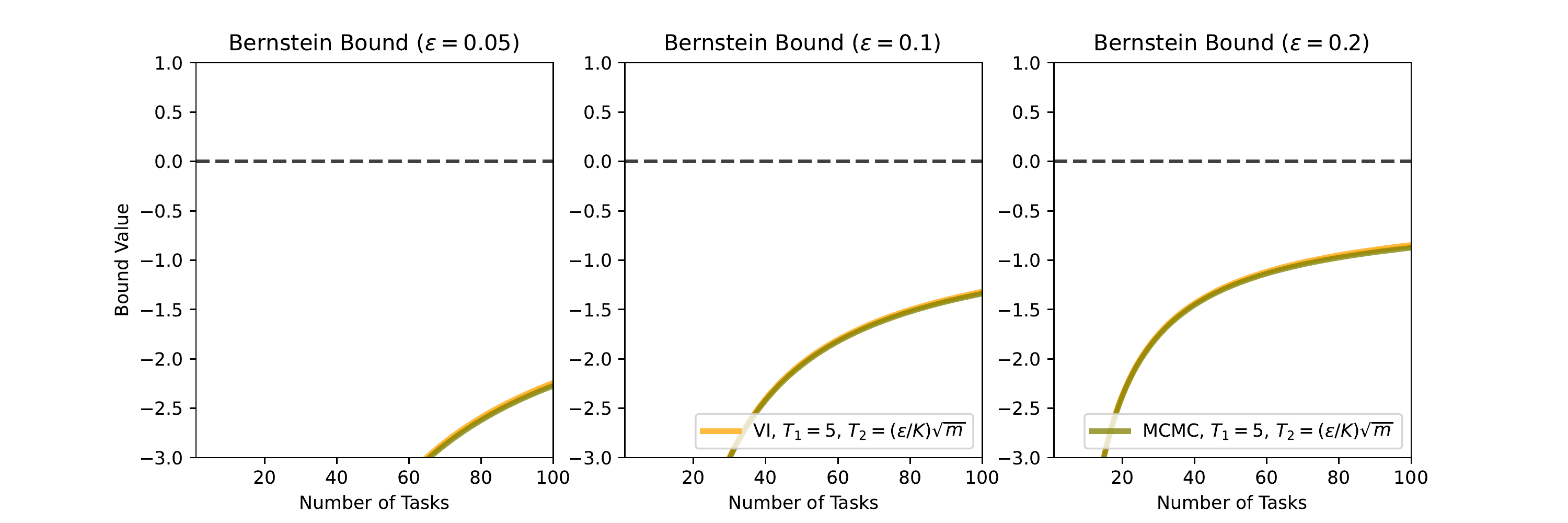}

\includegraphics[width=0.73\textwidth]{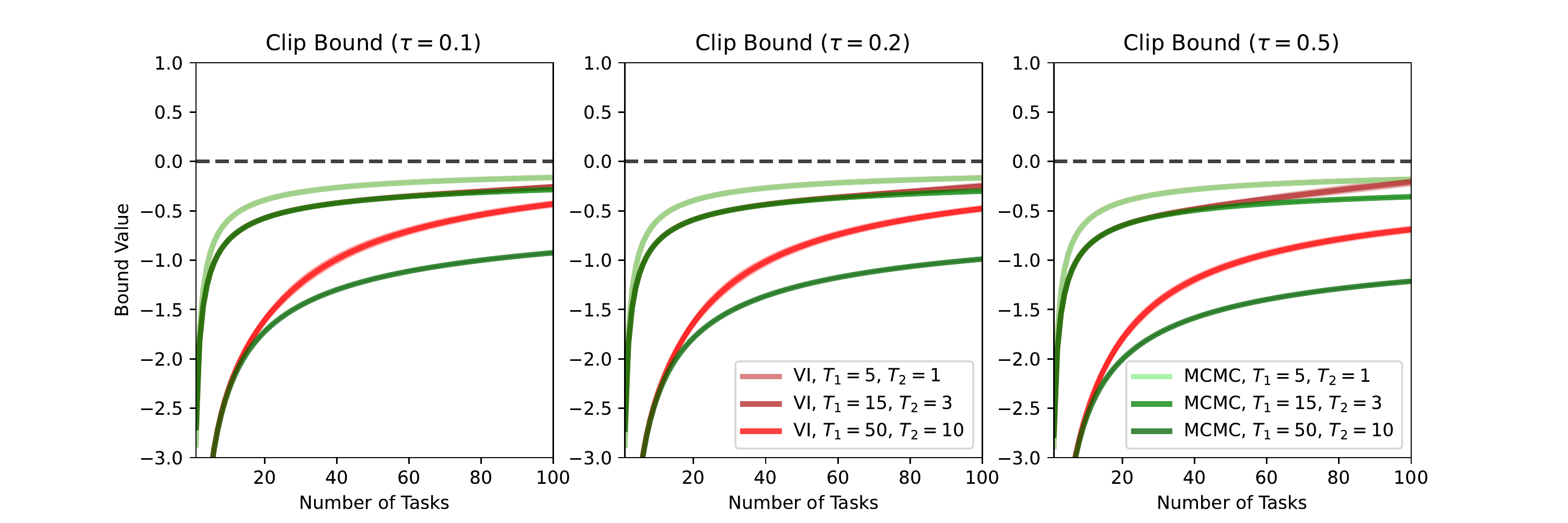}
\caption{The average reward obtained with the Bernstein bound (top row) and with the clipping bound (upper middle row) in tasks sampled from environment three. The values of the Bernstein bound (lower middle row) and the clipping bound (bottom row) are also shown. The solid lines show the mean average reward and the shaded regions show the mean $\pm$ 1 standard deviation.}
\label{fig:a21}
\end{figure}

In Figure \ref{fig:a10}, the Bernstein bound on the marginal transfer reward was below 0 for every value of $\epsilon$. Unsurprisingly, given the explicit dependence of the bound on $\epsilon$, increasing $\epsilon$ resulted in greater lower bound values. The clip bound was greater than 0 for PAC-Bayes VI with $T_1=15$ and $T_2=3$. For $\tau$ equal to 0.1, 0.2 and 0.5, the bound values at this temperature and after 100 tasks were 0.227, 0.238 and 0.202 respectively. Particularly at the highest temperatures, the clip bound value for the MCMC version was far below the clip bound value for the VI version. Figure \ref{fig:a20} shows the average reward and bound values obtained in environment two.

\begin{figure}[H]
\centering
\includegraphics[width=0.73\textwidth]{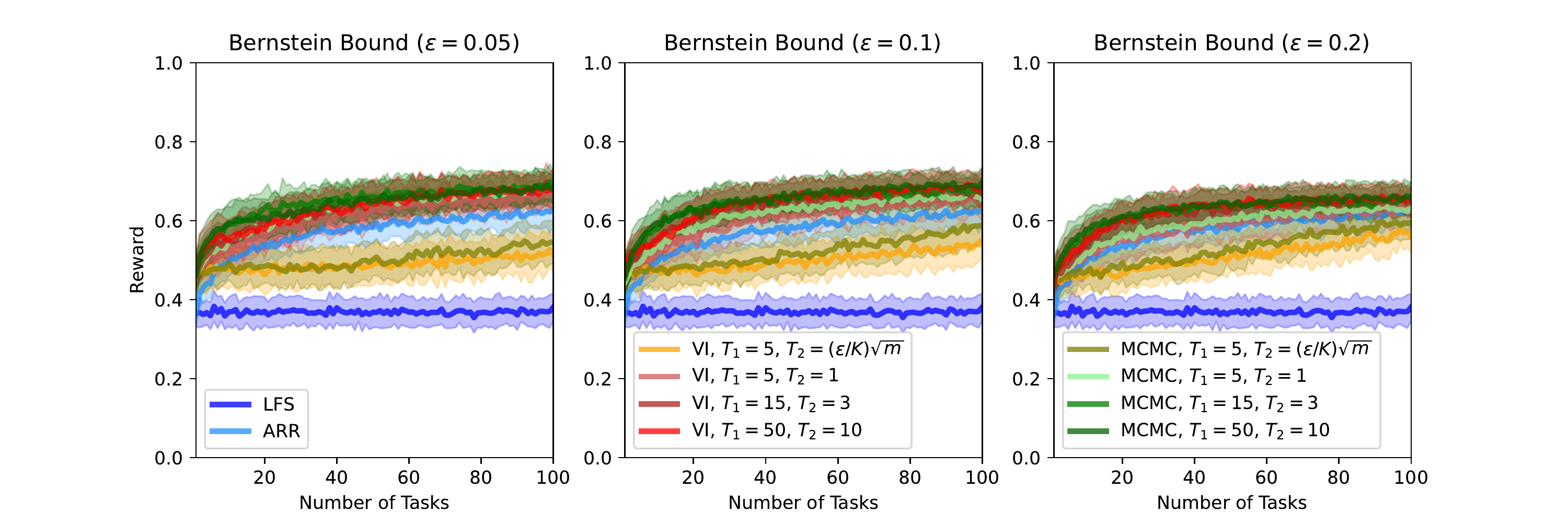}

\includegraphics[width=0.73\textwidth]{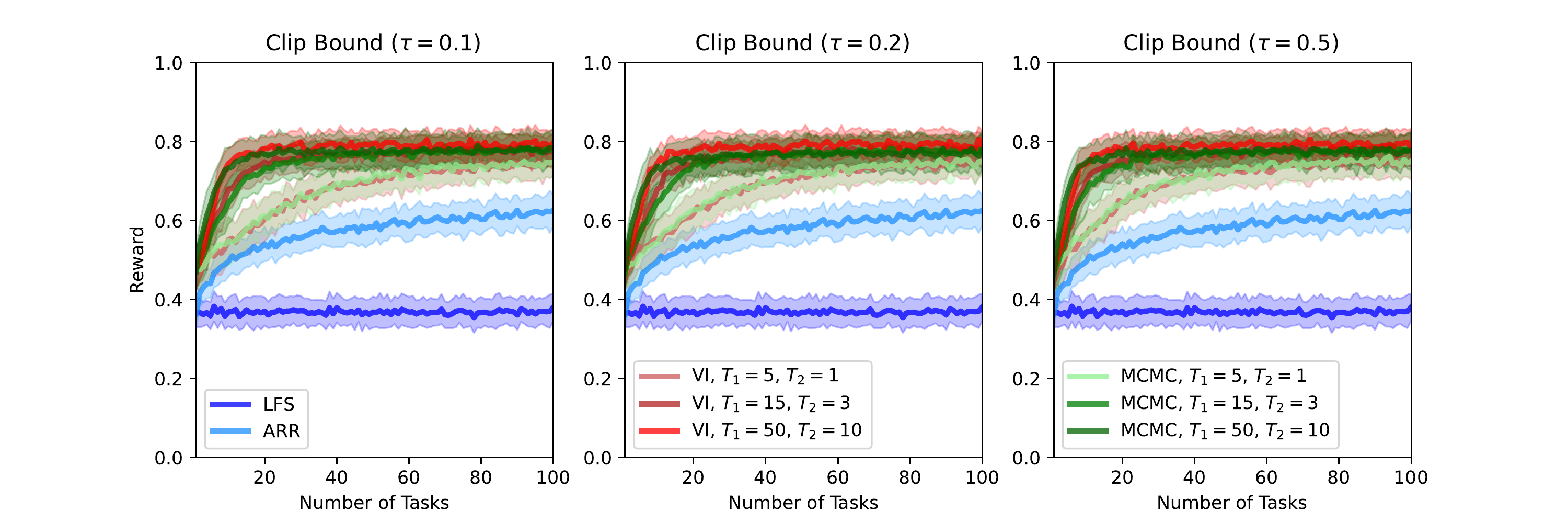}

\includegraphics[width=0.73\textwidth]{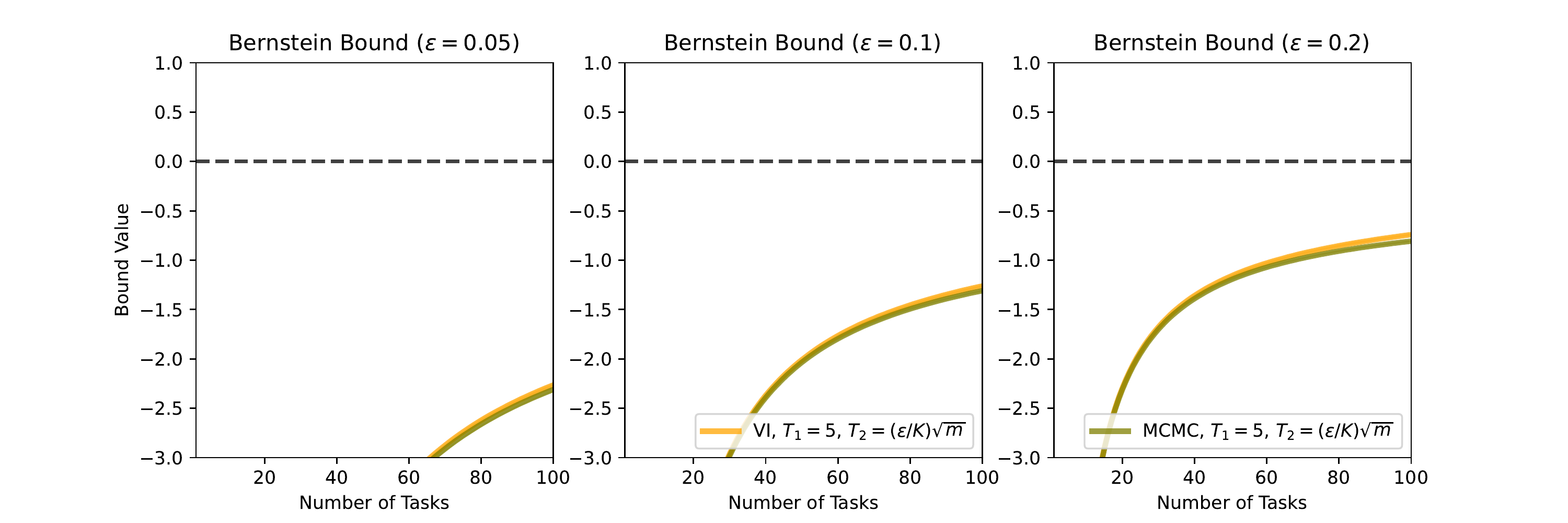}

\includegraphics[width=0.73\textwidth]{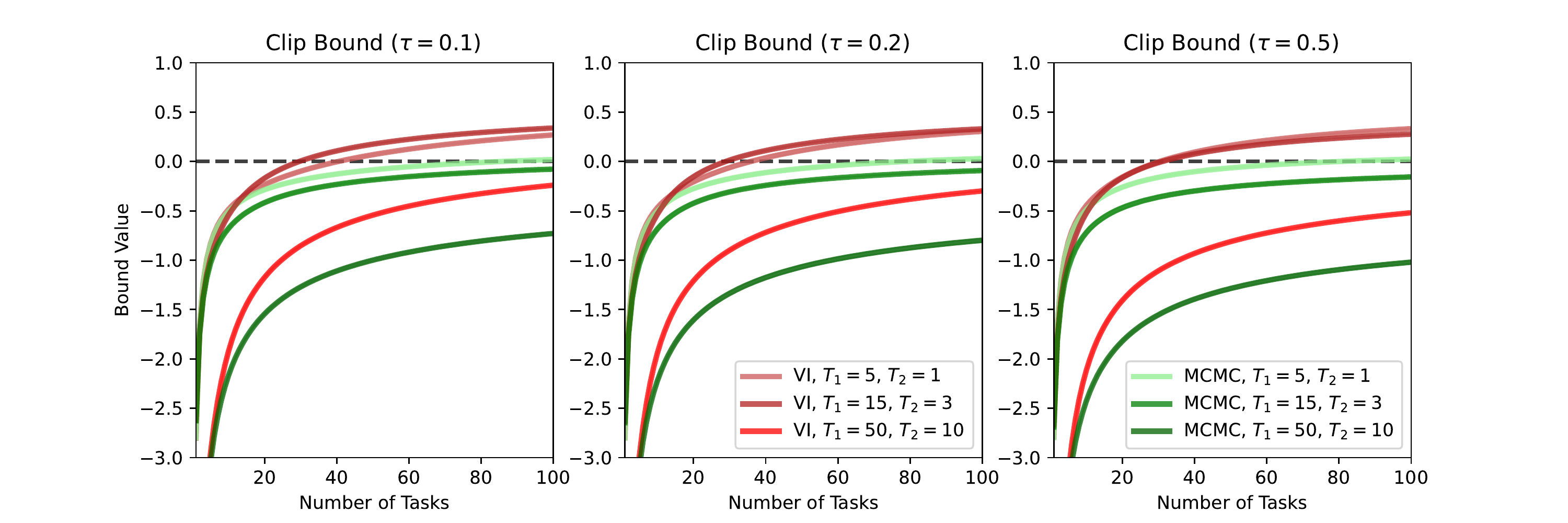}
\caption{The average reward obtained with the Bernstein bound (top row) and with the clipping bound (upper middle row) in tasks sampled from environment three, starting from a more informative hyperprior. The values of the Bernstein bound (lower middle row) and the clipping bound (bottom row) are also shown. The solid lines show the mean average reward and the shaded regions show the mean $\pm$ 1 standard deviation.}
\label{fig:a21_inf}
\end{figure}

Recall that in environment one (Figure \ref{fig:a10}), ARR reached higher average reward than our PAC-Bayes algorithms when the Bernstein bound was used. In environment two, this was no longer the case. In Figure \ref{fig:a10} ARR reached an average reward of 0.720. At the highest temperatures, and when $\epsilon$ was 0.05, PAC-Bayes VI and PAC-Bayes MCMC reached average rewards of 0.764 and 0.771 respectively. When the clip bound was used, PAC-Bayes VI with the highest temperatures and $\tau$ equal to $0.1$, $0.2$ and $0.5$ reached average rewards of 0.794, 0.796 and 0.784 respectively, which is almost the same as in environment one. At the highest temperature, PAC-Bayes MCMC was once again slightly worse, reaching only 0.778, 0.776 and 0.758 respectively for each value of $\tau$.

In Figure \ref{fig:a20}, the Bernstein bound was below 0 for every value of $\epsilon$. Unlike in environment one, the clip bound was below 0 for every temperature and every value of $\tau$. The lower bound values at the medium and lower temperatures were likely worse here than in environment one because the average rewards at these temperatures were lower in environment two. The highest bound value was -0.071, reached by PAC-Bayes VI after 100 tasks with $T_1=15$, $T_2=3$ and $\tau = 0.5$. Figure \ref{fig:a21} shows the average reward and bound values obtained in environment three.

In Figure \ref{fig:a21}, ARR reached an average reward of 0.626 by task 100. When the Bernstein bound was used and $\epsilon$ was 0.05, PAC-Bayes VI and PAC-Bayes MCMC with the highest temperatures reached 0.689 and 0.691 respectively. As in the previous two environments, when the clip bound was used, PAC-Bayes VI reached higher average reward than PAC-Bayes MCMC at the highest temperatures. With the highest temperatures and $\tau$ equal to $0.1$, $0.2$ and $0.5$, PAC-Bayes VI reached average rewards of 0.715, 0.703 and 0.710 respectively. PAC-Bayes MCMC reached average rewards of 0.666, 0.645 and 0.629 respectively. In the bottom two rows of Figure \ref{fig:a21}, the Bernstein lower bound was once again below 0 for every value of $\epsilon$. The clip lower bound was below 0 for every temperature and every value of $\tau$.

Lastly, we tested our algorithms once more in environment three, but this time with a more informative hyperprior. The hyperprior is still a diagonal Gaussian distribution over the weight vector of a softmax policy, but this time the last element of the mean vector is 2 instead of 0. Therefore, priors sampled from this informative hyperprior are likely to assign a higher probability to the last action than the other actions. Since the last action is usually the best action, for tasks sampled from environment three, we expect this hyperprior to improve the average reward and bound values obtained. Our results with this informative hyperprior are shown in Figure \ref{fig:a21_inf}.

In Figure \ref{fig:a21_inf}, when the Bernstein bound was used and $\epsilon$ was 0.05, PAC-Bayes VI and PAC-Bayes MCMC with the highest temperatures reached an average rewards of 0.681 and 0.678 respectively. These average rewards with the informative hyperprior are almost no different from the corresponding average rewards in Figure \ref{fig:a21}, where the uninformative hyperprior was used. However, when the clip bound was used, both PAC-Bayes VI and PAC-Bayes MCMC reached higher average rewards with the informative hyperprior. With the highest temperatures and $\tau$ equal to $0.1$, $0.2$ and $0.5$, PAC-Bayes VI reached average rewards of 0.797, 0.792 and 0.787 respectively. PAC-Bayes MCMC reached average rewards of 0.785, 0.778 and 0.777 respectively.

In the bottom two rows of Figure \ref{fig:a21_inf}, even when using the informative hyperprior, the Bernstein bound was below 0 for every value of $\epsilon$. However, the clip bound was above 0 when PAC-Bayes VI was run with the medium ($T_1=15, T_2=3$) or low ($T_1=5, T_2=1$) temperatures. At the medium temperatures, the bound values after 100 tasks when $\tau$ was $0.1$, $0.2$ and $0.5$ were $0.339$, $0.331$ and $0.277$.

Our results suggest that, when run with high enough values of the temperature parameters, our PAC-Bayes algorithms can obtain higher average reward than ARR. PAC-Bayes VI with the clip bound consistently reached the highest average reward when run with the highest temperatures and any value of $\tau$. PAC-Bayes VI with the clip bound also produced the best bound values. In environment one and environment three with the informative hyperprior, it was possible to obtain close to the best average reward and a non-trivial lower bound value simultaneously by running PAC-Bayes VI with the clip bound and with $T_1=15, T_2=3$. In general, higher temperatures yielded higher average reward whereas the best bound values were achieved with medium or low temperatures. This suggests that using the bound value to select the temperature parameters, as discussed in Sect.\ \ref{sec:alg}, may result in sub-optimal average reward.

\section{Related Work}

Lifelong learning is related to problems in which the goal is to learn from a set of tasks. In multi-task learning \cite{caruana1997multitask}, the goal is to learn multiple tasks simultaneously. In domain adaptation \cite{ben2010theory}, the goal is to learn one task, using data from other tasks. In meta learning \cite{schmidhuber1987evolutionary}, the goal is to learn a learning algorithm from a set of tasks that performs well on new tasks. Approaches that use PAC-Bayesian analysis for these problems are closely related to this work. For example, PAC-Bayesian methods for domain adaptation \cite{germain2013domain} \cite{germain2016new} \cite{germain2020domain} and meta learning \cite{amit2018meta} \cite{rothfuss2021pacoh}.

The first PAC-Bayesian lifelong learning bound \cite{pentina2014pac} has since been extended to certain situations where tasks are not sampled i.i.d. from a task environment \cite{pentina2015lifelong}. This extension considers the case where the task data distributions are sampled from the same distribution, but not independently, and also the case where tasks are sampled independently, but the task environment changes over time.

PAC-Bayesian bounds have previously been derived for single multi-armed bandit problems \cite{seldin2011pacmab} \cite{seldin2012bern} and contextual bandit problems \cite{seldin2011paccb}. These results made heavy use of PAC-Bayesian martingale concentration inequalities \cite{seldin2012mart}. Subsequent PAC-Bayesian inequalities for martingales \cite{wang2015pac} \cite{balsubramani2015pac} may therefore also prove useful for studying the multi-armed bandit and contextual bandit problems.

Several works have applied PAC-Bayesian analysis to offline bandit problems. In offline bandit problems, the training data are all sampled from a single, fixed behaviour policy. This means that action-reward pairs are independent of each other. PAC-Bayesian bounds for the clipped importance-weighted reward estimate \cite{london2019bayesian} and for a self-normalised weighted importance sampling estimate \cite{kuzborskij2019efron} have been derived in the offline setting.

PAC-Bayesian bounds for model selection \cite{fard2010pac} and policy evaluation \cite{fard2011pac} have been derived in the reinforcement learning setting. Finally, PAC-Bayesian bounds have been derived for out-of-distribution generalisation \cite{majumdar2018pac}, \cite{veer2020probably} \cite{ren2020generalization} and out-of-distribution detection \cite{farid2021task} in reinforcement learning problems.

\section{Conclusion}\label{sect:conclusion}

In this paper, we derived the first PAC-Bayesian generalisation bounds for lifelong learning of MAB tasks. We proposed lifelong learning algorithms, PAC-Bayes VI and PAC-Bayes MCMC, that use our bounds as their learning objectives. In our experiments, we found that when run with high enough values of their temperature parameters, our algorithms performed better than ARR. When run with lower values of the temperature parameters, our clipping bound gave non-trivial lower bounds on the marginal transfer reward.

We will conclude by discussing some limitations of our bounds and some ideas for future work. Our bounds use the importance-weighted empirical reward estimate, which is known to have high variance. This means it is difficult to derive tight bounds. This problem is somewhat addressed by clipping the importance weights. However, other reward estimates with reduced variance, such as weighted importance sampling \cite{rubinstein1981simulation} and the doubly robust estimator \cite{dudik2014doubly}, may allow for tighter bounds. We therefore see exploring alternative reward estimates as a direction for future work.

Our bounds contain a term, $c_n$, that cannot easily be computed. Since this term is constant with respect to the hyperposterior $\mathcal{Q}$, it does not effect our algorithms that learn a hyperposterior by maximising our bounds. However, this is a problem if we want to evaluate our lower bounds. Therefore, exploring alternative proof techniques in order to replace the unknown $c_n$ term is another direction for future work.

\bibliography{biblio}

\begin{thebibliography}{10}

\bibitem{amit2018meta}
R.~Amit and R.~Meir.
\newblock Meta-learning by adjusting priors based on extended pac-bayes theory.
\newblock In {\em International Conference on Machine Learning}, pages
  205--214. PMLR, 2018.

\bibitem{azar2013sequential}
M.~G. Azar, A.~Lazaric, and E.~Brunskill.
\newblock Sequential transfer in multi-armed bandit with finite set of models.
\newblock In {\em Advances in Neural Information Processing Systems}, 2013.

\bibitem{azuma1967weighted}
K.~Azuma.
\newblock Weighted sums of certain dependent random variables.
\newblock {\em Tohoku Mathematical Journal, Second Series}, 19(3):357--367,
  1967.

\bibitem{balsubramani2015pac}
A.~Balsubramani.
\newblock Pac-bayes iterated logarithm bounds for martingale mixtures.
\newblock {\em arXiv preprint arXiv:1506.06573}, 2015.

\bibitem{banerjee2006bayesian}
A.~Banerjee.
\newblock On bayesian bounds.
\newblock In {\em Proceedings of the 23rd international conference on Machine
  learning}, pages 81--88, 2006.

\bibitem{baxter2000model}
J.~Baxter.
\newblock A model of inductive bias learning.
\newblock {\em Journal of artificial intelligence research}, 12:149--198, 2000.

\bibitem{ben2010theory}
S.~Ben-David, J.~Blitzer, K.~Crammer, A.~Kulesza, F.~Pereira, and J.~W.
  Vaughan.
\newblock A theory of learning from different domains.
\newblock {\em Machine learning}, 79(1-2):151--175, 2010.

\bibitem{brunskill2013sample}
E.~Brunskill and L.~Li.
\newblock Sample complexity of multi-task reinforcement learning.
\newblock In {\em Proceedings of the Twenty-Ninth Conference on Uncertainty in
  Artificial Intelligence}, 2013.

\bibitem{burda2016importance}
Y.~Burda, R.~Grosse, and R.~Salakhutdinov.
\newblock Importance weighted autoencoders.
\newblock In {\em International Conference on Learning Representations}, 2016.

\bibitem{caruana1997multitask}
R.~Caruana.
\newblock Multitask learning.
\newblock {\em Machine learning}, 28(1):41--75, 1997.

\bibitem{catoni2004statistical}
O.~Catoni.
\newblock {\em Statistical learning theory and stochastic optimization: Ecole
  d'Et{\'e} de Probabilit{\'e}s de Saint-Flour, XXXI-2001}, volume 1851.
\newblock Springer Science \& Business Media, 2004.

\bibitem{cesa2006prediction}
N.~Cesa-Bianchi and G.~Lugosi.
\newblock {\em Prediction, learning, and games}.
\newblock Cambridge university press, 2006.

\bibitem{deramo2020sharing}
C.~D'Eramo, D.~Tateo, A.~Bonarini, M.~Restelli, and J.~Peters.
\newblock Sharing knowledge in multi-task deep reinforcement learning.
\newblock In {\em International Conference on Learning Representations (ICLR)},
  2020.

\bibitem{donsker1975asymptotic}
M.~D. Donsker and S.~S. Varadhan.
\newblock Asymptotic evaluation of certain markov process expectations for
  large time, i.
\newblock {\em Communications on Pure and Applied Mathematics}, 28(1):1--47,
  1975.

\bibitem{dudik2014doubly}
M.~Dud{\'\i}k, D.~Erhan, J.~Langford, and L.~Li.
\newblock Doubly robust policy evaluation and optimization.
\newblock {\em Statistical Science}, 29(4):485--511, 2014.

\bibitem{fard2010pac}
M.~Fard and J.~Pineau.
\newblock Pac-bayesian model selection for reinforcement learning.
\newblock {\em Advances in Neural Information Processing Systems},
  23:1624--1632, 2010.

\bibitem{fard2011pac}
M.~M. Fard, J.~Pineau, and C.~Szepesv{\'a}ri.
\newblock Pac-bayesian policy evaluation for reinforcement learning.
\newblock In {\em Proceedings of the Twenty-Seventh Conference on Uncertainty
  in Artificial Intelligence}, pages 195--202, 2011.

\bibitem{farid2021task}
A.~Farid, S.~Veer, and A.~Majumdar.
\newblock Task-driven out-of-distribution detection with statistical guarantees
  for robot learning.
\newblock In {\em 5th Annual Conference on Robot Learning}, 2021.

\bibitem{germain2013domain}
P.~Germain, A.~Habrard, F.~Laviolette, and E.~Morvant.
\newblock A pac-bayesian approach for domain adaptation with specialization to
  linear classifiers.
\newblock In {\em International conference on machine learning}, pages
  738--746. PMLR, 2013.

\bibitem{germain2016new}
P.~Germain, A.~Habrard, F.~Laviolette, and E.~Morvant.
\newblock A new pac-bayesian perspective on domain adaptation.
\newblock In {\em International conference on machine learning}, pages
  859--868. PMLR, 2016.

\bibitem{germain2020domain}
P.~Germain, A.~Habrard, F.~Laviolette, and E.~Morvant.
\newblock Pac-bayes and domain adaptation.
\newblock {\em Neurocomputing}, 379:379--397, 2020.

\bibitem{guedj2019primer}
B.~Guedj and J.~Shawe-Taylor.
\newblock A primer on pac-bayesian learning.
\newblock In {\em ICML 2019-Thirty-sixth International Conference on Machine
  Learning}, 2019.

\bibitem{hoeffding1994probability}
W.~Hoeffding.
\newblock Probability inequalities for sums of bounded random variables.
\newblock In {\em The Collected Works of Wassily Hoeffding}, pages 409--426.
  Springer, 1994.

\bibitem{kingma2014adam}
D.~P. Kingma and J.~Ba.
\newblock Adam: A method for stochastic optimization.
\newblock {\em arXiv preprint arXiv:1412.6980}, 2014.

\bibitem{kingma2013auto}
D.~P. Kingma and M.~Welling.
\newblock Auto-encoding variational bayes.
\newblock {\em arXiv preprint arXiv:1312.6114}, 2013.

\bibitem{kuzborskij2019efron}
I.~Kuzborskij and C.~Szepesv{\'a}ri.
\newblock Efron-stein pac-bayesian inequalities.
\newblock {\em arXiv preprint arXiv:1909.01931}, 2019.

\bibitem{lazaric2011transfer}
A.~Lazaric and M.~Restelli.
\newblock Transfer from multiple mdps.
\newblock In {\em Advances in Neural Information Processing Systems}, pages
  1746–--1754, 2011.

\bibitem{li2016preconditioned}
C.~Li, C.~Chen, D.~Carlson, and L.~Carin.
\newblock Preconditioned stochastic gradient langevin dynamics for deep neural
  networks.
\newblock In {\em Proceedings of the AAAI Conference on Artificial
  Intelligence}, volume~30, 2016.

\bibitem{london2019bayesian}
B.~London and T.~Sandler.
\newblock Bayesian counterfactual risk minimization.
\newblock In {\em International Conference on Machine Learning}, pages
  4125--4133. PMLR, 2019.

\bibitem{majumdar2018pac}
A.~Majumdar and M.~Goldstein.
\newblock Pac-bayes control: synthesizing controllers that provably generalize
  to novel environments.
\newblock In {\em Conference on Robot Learning}, pages 293--305, 2018.

\bibitem{mcallester1998some}
D.~A. McAllester.
\newblock Some pac-bayesian theorems.
\newblock In {\em Proceedings of the International Conference on Computational
  Learning Theory (COLT)}, 1998.

\bibitem{pentina2014pac}
A.~Pentina and C.~Lampert.
\newblock A {PAC}-{B}ayesian bound for lifelong learning.
\newblock In {\em International Conference on Machine Learning}, pages
  991--999, 2014.

\bibitem{pentina2015lifelong}
A.~Pentina and C.~H. Lampert.
\newblock Lifelong learning with non-iid tasks.
\newblock {\em Advances in Neural Information Processing Systems},
  28:1540--1548, 2015.

\bibitem{ren2020generalization}
A.~Ren, S.~Veer, and A.~Majumdar.
\newblock Generalization guarantees for imitation learning.
\newblock In {\em Conference on Robot Learning}, 2020.

\bibitem{rothfuss2021pacoh}
J.~Rothfuss, V.~Fortuin, M.~Josifoski, and A.~Krause.
\newblock Pacoh: Bayes-optimal meta-learning with pac-guarantees.
\newblock In {\em International Conference on Machine Learning}, pages
  9116--9126. PMLR, 2021.

\bibitem{rubinstein1981simulation}
R.~Y. Rubinstein.
\newblock Simulation and the monte carlo method.
\newblock Technical report, 1981.

\bibitem{schmidhuber1987evolutionary}
J.~Schmidhuber.
\newblock {\em Evolutionary principles in self-referential learning.}
\newblock PhD thesis, Technische Universit{\"a}t M{\"u}nchen, 1987.

\bibitem{seldin2011paccb}
Y.~Seldin, P.~Auer, J.~Shawe-taylor, R.~Ortner, and F.~Laviolette.
\newblock {PAC}-{B}ayesian analysis of contextual bandits.
\newblock {\em Advances in neural information processing systems},
  24:1683--1691, 2011.

\bibitem{seldin2012bern}
Y.~Seldin, N.~Cesa-Bianchi, P.~Auer, F.~Laviolette, and J.~Shawe-Taylor.
\newblock {PAC}-{B}ayes-{B}ernstein inequality for martingales and its
  application to multiarmed bandits.
\newblock {\em JMLR Workshop and Conference Proceedings}, 26:98--111, 2012.

\bibitem{seldin2012mart}
Y.~Seldin, F.~Laviolette, N.~Cesa-Bianchi, J.~Shawe-Taylor, and P.~Auer.
\newblock {PAC}-{B}ayesian inequalities for martingales.
\newblock {\em IEEE Transactions on Information Theory}, 58(12):7086--7093,
  2012.

\bibitem{seldin2011pacmab}
Y.~Seldin, F.~Laviolette, J.~Shawe-Taylor, J.~Peters, and P.~Auer.
\newblock {PAC}-{B}ayesian analysis of martingales and multiarmed bandits.
\newblock {\em arXiv preprint arXiv:1105.2416}, 2011.

\bibitem{shawe1997pac}
J.~Shawe-Taylor and R.~C. Williamson.
\newblock A pac analysis of a bayesian estimator.
\newblock In {\em Proceedings of the tenth annual conference on Computational
  learning theory}, pages 2--9, 1997.

\bibitem{soare2014multi}
M.~Soare, O.~Alsharif, A.~Lazaric, and J.~Pineau.
\newblock Multi-task linear bandits.
\newblock In {\em NIPS2014 Workshop on Transfer and Multi-task Learning: Theory
  meets Practice}, 2014.

\bibitem{thrun1995lifelong}
S.~Thrun and T.~M. Mitchell.
\newblock Lifelong robot learning.
\newblock {\em Robotics and autonomous systems}, 15(1-2):25--46, 1995.

\bibitem{veer2020probably}
S.~Veer and A.~Majumdar.
\newblock Probably approximately correct vision-based planning using motion
  primitives.
\newblock In {\em Conference on Robot Learning}, 2020.

\bibitem{wang2015pac}
Z.~Wang, L.~Shen, Y.~Miao, S.~Chen, and W.~Xu.
\newblock Pac-bayesian inequalities of some random variables sequences.
\newblock {\em Journal of Inequalities and Applications}, 2015(1):1--8, 2015.

\end{thebibliography}
\bibliographystyle{abbrv}

\appendix

\section{Additional Proofs}

\subsection{Proof of Lemma \ref{lem:azuma}}

\begin{proof}[Proof of Lemma \ref{lem:azuma}]

We have that:

\begin{align*}
\mathop{\mathbb{E}}_{X_1, \dots, X_n}\left[e^{\lambda\sum_{i=1}^{n}X_i}\right] &= \mathop{\mathbb{E}}_{X_1, \dots, X_n}\left[\prod_{i=1}^{n}e^{\lambda X_i}\right]\\
&= \mathop{\mathbb{E}}_{X_1, \dots, X_{n-1}}\left[\mathop{\mathbb{E}}_{X_n}\condexp*{\prod_{i=1}^{n}e^{\lambda X_i}}{X_1, \dots, X_{n-1}}\right]\\
&\leq \mathop{\mathbb{E}}_{X_1, \dots, X_{n-1}}\left[e^{\lambda\mathop{\mathbb{E}}_{X_n}[X_n|X_1, \dots, X_{n-1}]}e^{\frac{\lambda^2}{8}(b_n - a_n)^2}\prod_{i=1}^{n-1}e^{\lambda X_i}\right]\\
&\leq e^{\frac{\lambda^2}{8}(b_n - a_n)^2}\mathop{\mathbb{E}}_{X_1, \dots, X_{n-1}}\left[e^{\lambda\sum_{i=1}^{n-1}X_i}\right]
\end{align*}

By iterating the above steps, we have that:

\begin{equation*}
\mathop{\mathbb{E}}_{X_1, \dots, X_n}\left[e^{\lambda\sum_{i=1}^{n}X_i}\right] \leq \prod_{i=1}^{n}e^{\frac{\lambda^2}{8}(b_i - a_i)^2} = e^{\frac{\lambda^2}{8}\sum_{i=1}^{n}(b_i - a_i)^2}
\end{equation*}
\end{proof}

\subsection{Proof of Lemma \ref{lem:multi_risk_trunc}}

\begin{proof}[Proof of Lemma \ref{lem:multi_risk_trunc}]

Let $M_i^j(a) = \sum_{k=1}^{j}X_{ik}^{\tau}(a)$, where $X_{ik}^{\tau}(a)$ is as defined in Equation (\ref{eqn:supermax}). Using the compression lemma, we have that for any $k = 1, \dots, m$

\begin{align}
\mathop{\mathbb{E}}_{P \sim \mathcal{Q}, a_i \sim A(D_i^{:k-1},P)}\left[\sum_{i=1}^{n}M_i^m(a_i)\right] &\leq \frac{1}{\lambda}D_{\mathrm{KL}}((\mathcal{Q}, A_{k-1}^n)||(\mathcal{P}, P^n))\nonumber\\
&+\frac{1}{\lambda}\mathrm{ln}\left(\mathop{\mathbb{E}}_{P \sim \mathcal{P}, a_i \sim P}\left[e^{\sum_{i=1}^{n}M_i^m(a_i)}\right]\right)\label{eqn:lem2:compress}.
\end{align}

Now, we need to upper bound the term inside the logarithm. For any $\delta \in (0, 1]$, using Markov's inequality with respect to expectations over $D_1, \dots, D_n$, the following inequality holds with probability greater than $1 - \delta$

\begin{equation*}
\mathop{\mathbb{E}}_{P \sim \mathcal{P}, a_i \sim P}\left[\prod_{i=1}^{n}e^{\lambda M_i^m(a_i)}\right] \leq \frac{1}{\delta}\mathop{\mathbb{E}}_{D_1, \dots, D_n}\mathop{\mathbb{E}}_{P \sim \mathcal{P}, a_i \sim P}\left[\prod_{i=1}^{n}e^{\lambda M_i^m(a_i)}\right].
\end{equation*}

If $\mathcal{P}$ does not depend on any of the training sets $D_1, \dots, D_n$, then the order of the expectations can be swapped.

\begin{equation*}
\frac{1}{\delta}\mathop{\mathbb{E}}_{D_1, \dots, D_n}\mathop{\mathbb{E}}_{P \sim \mathcal{P}, a_i \sim P}\left[\prod_{i=1}^{n}e^{\lambda M_i^m(a_i)}\right] = \frac{1}{\delta}\mathop{\mathbb{E}}_{P \sim \mathcal{P}, a_i \sim P}\mathop{\mathbb{E}}_{D_1, \dots, D_n}\left[\prod_{i=1}^{n}e^{\lambda M_i^m(a_i)}\right].
\end{equation*}

Since each training set is only dependent on the training sets that came before it, the expectation over $D_1, \dots, D_n$ can be factorised.

\begin{align*}
\frac{1}{\delta}&\mathop{\mathbb{E}}_{P \sim \mathcal{P}, a_i \sim P}\mathop{\mathbb{E}}_{D_1, \dots, D_n}\left[\prod_{i=1}^{n}e^{\lambda M_i^m(a_i)}\right]\\
&= \frac{1}{\delta}\mathop{\mathbb{E}}_{P \sim \mathcal{P}, a_i \sim P}\mathop{\mathbb{E}}_{D_1, \dots, D_{n-1}}\left[\mathop{\mathbb{E}}_{D_n}\condexp*{\prod_{i=1}^{n}e^{\lambda M_i^m(a_i)}}{D_1, \dots, D_{n-1}}\right].
\end{align*}

Using Lemma \ref{lem:azuma}, the $n$th term of the product can be upper bounded.

\begin{align*}
\frac{1}{\delta}\mathop{\mathbb{E}}_{P \sim \mathcal{P}, a_i \sim P}&\mathop{\mathbb{E}}_{D_1, \dots, D_{n-1}}\left[\mathop{\mathbb{E}}_{D_n}\condexp*{\prod_{i=1}^{n}e^{\lambda M_i^m(a_i)}}{D_1, \dots, D_{n-1}}\right]\\
&\leq \frac{1}{\delta}\mathop{\mathbb{E}}_{P \sim \mathcal{P}, a_i \sim P}\mathop{\mathbb{E}}_{D_1, \dots, D_{n-1}}\left[e^{\frac{m\lambda^2(1+\tau)^2}{8}}\prod_{i=1}^{n-1}e^{\lambda M_i^m(a_i)}\right]
\end{align*}

Through alternating steps of factorisation and application of Lemma \ref{lem:azuma}, we have that

\begin{equation*}
\frac{1}{\delta}\mathop{\mathbb{E}}_{P \sim \mathcal{P}, a_i \sim P}\mathop{\mathbb{E}}_{D_1, \dots, D_n}\left[\prod_{i=1}^{n}e^{\lambda M_i^m(a_i)}\right] \leq \frac{1}{\delta}e^{\frac{nm\lambda^2(1+\tau)^2}{8}}.
\end{equation*}

Substituting this into Equation (\ref{eqn:lem2:compress}), we have that for any $\delta \in (0, 1]$, any $\lambda > 0$, any $\tau > 0$ and any $k \in \{1, \dots, m\}$, the following inequality holds with probability at least $1 - \delta$.

\begin{align}
\label{eqn:pre_union}
\mathop{\mathbb{E}}_{P \sim \mathcal{Q}, a_i \sim A(D_i^{:k-1},P)}\left[\sum_{i=1}^{n}M_i^m(a_i)\right] &\leq \frac{1}{\lambda}D_{\mathrm{KL}}((\mathcal{Q}, A_{k-1}^n)||(\mathcal{P}, P^n))\\
&+ \frac{nm\lambda(1+\tau)^2}{8} + \frac{1}{\lambda}\mathrm{ln}(1/\delta)\nonumber
\end{align}

Using the union bound, if we replace $\mathrm{ln}(1/\delta)$ with $\mathrm{ln}(m/\delta)$, then Equation \ref{eqn:pre_union} holds simultaneously for all $k = 1, \dots, m$ with probability at least $1 - \delta$. From the definitions of $\widetilde{\mathcal{R}}(\mathcal{Q})$, $\widehat{\mathcal{R}}_{\tau}(\mathcal{Q})$ and $M_i^j(a)$:

\begin{align*}
\widehat{\mathcal{R}}_{\tau}(\mathcal{Q}) - \widetilde{\mathcal{R}}(\mathcal{Q}) &= \frac{1}{m}\sum_{k=1}^{m}\frac{1}{nm}\mathop{\mathbb{E}}_{P \sim \mathcal{Q}, a_i \sim A(D_i^{:k-1},P)}\left[\sum_{i=1}^{n}M_i^m(a_i)\right]
\end{align*}

Substituting in the result of Equation \ref{eqn:pre_union}, we have that with probability at least $1 - \delta$:

\begin{align*}
\widehat{\mathcal{R}}_{\tau}(\mathcal{Q}) - \widetilde{\mathcal{R}}(\mathcal{Q}) &\leq \frac{1}{m}\sum_{k=1}^{m}\frac{1}{nm}\frac{1}{\lambda}D_{\mathrm{KL}}((\mathcal{Q}, A_{k-1}^n)||(\mathcal{P}, P^n))\\
&+ \frac{\lambda(1+\tau)^2}{8} + \frac{1}{nm\lambda}\mathrm{ln}(m/\delta)
\end{align*}

By using Equation \ref{eqn:kl} and rearranging this inequality, we obtain:

\begin{align*}
\widetilde{\mathcal{R}}(\mathcal{Q}) &\geq \widehat{\mathcal{R}}_{\tau}(\mathcal{Q}) - \frac{1}{nm\lambda}D_{\mathrm{KL}}(\mathcal{Q}||\mathcal{P}) - \frac{1}{nm^2\lambda}\mathop{\mathbb{E}}_{P \sim \mathcal{Q}}\left[\sum_{i=1}^{n}\sum_{j=1}^{m}D_{\mathrm{KL}}(A(D_i^{:j-1}, P)||P)\right]\\
&- \frac{\lambda(1+\tau)^2}{8} - \frac{1}{nm\lambda}\mathrm{ln}(m/\delta)
\end{align*}

Finally, the substitution $\lambda^{\prime} = \lambda/m$ yields the statement of the lemma.

\end{proof}

\subsection{Proof of Theorem \ref{thm:big_bound_trunc}}

\begin{proof}[Proof of Theorem \ref{thm:big_bound_trunc}]

By Lemma \ref{lem:transfer_bound}, for any hyperprior $\mathcal{P}$, $\lambda_1 > 0$ and any $\delta_1 \in (0, 1]$, the following inequality holds with probability at least $1 - \delta_1$

\begin{equation*}
\mathcal{R}(\mathcal{Q}) \geq \widetilde{\mathcal{R}}(\mathcal{Q}) - \frac{1}{\lambda_1}D_{\mathrm{KL}}(\mathcal{Q}||\mathcal{P}) - c_n - \frac{\lambda_1}{8n} - \frac{1}{\lambda_1}\mathrm{ln}(1/\delta_1).
\end{equation*}

By Lemma \ref{lem:multi_risk_trunc}, for any hyperprior $\mathcal{P}$, any $\lambda_2 > 0$, any $\tau > 0$ and any $\delta_2 \in (0, 1]$, the following inequality holds with probability at least $1 - \delta_2$

\begin{align*}
\widetilde{\mathcal{R}}(\mathcal{Q}) &\geq \widehat{\mathcal{R}}_{\tau}(\mathcal{Q}) - \frac{1}{n\lambda_2}D_{\mathrm{KL}}(\mathcal{Q}||\mathcal{P}) - \frac{1}{nm\lambda_2}\mathop{\mathbb{E}}_{P \sim \mathcal{Q}}\left[\sum_{i=1}^{n}\sum_{j=1}^{m}D_{\mathrm{KL}}(A(D_i^{:j}, P)||P)\right]\\
&- \frac{\lambda_2(1+\tau)^2}{8m} - \frac{1}{n\lambda_2}\mathrm{ln}(m/\delta_2)
\end{align*}

By the union bound, the probability that both inequalities hold simultaneously is at least $1 - \delta_1 - \delta_2$. Therefore, if we set $\delta_1 = \delta_2 = \delta/2$, we have that with probability at least $1 - \delta$

\begin{align*}
\mathcal{R}(\mathcal{Q}) &\geq \widehat{\mathcal{R}}_{\tau}(\mathcal{Q}) - \left(\frac{1}{\lambda_1} + \frac{1}{n\lambda_2}\right)D_{\mathrm{KL}}(\mathcal{Q}||\mathcal{P})\\
&- \frac{1}{nm\lambda_2}\mathop{\mathbb{E}}_{P \sim \mathcal{Q}}\left[\sum_{i=1}^{n}\sum_{j=1}^{m}D_{\mathrm{KL}}(A(D_{i}^{:j-1}, P)||P)\right]\\
&- c_n - \frac{\lambda_1}{8n} - \frac{\lambda_2(1+\tau)^2}{8m} - \frac{1}{\lambda_1}\mathrm{ln}(2/\delta) - \frac{1}{n\lambda_2}\mathrm{ln}(2m/\delta).
\end{align*}

\end{proof}

\end{document}